\documentclass{article} % For LaTeX2e
\usepackage{iftex}
\ifPDFTeX
  \usepackage{iclr2027_conference,times}
\else
  \usepackage{iclr2027_conference}
\fi

\usepackage{amsmath,amsfonts,bm}

\def\eqref#1{equation~\ref{#1}}
\def\1{\bm{1}}

\DeclareMathAlphabet{\mathsfit}{\encodingdefault}{\sfdefault}{m}{sl}
\SetMathAlphabet{\mathsfit}{bold}{\encodingdefault}{\sfdefault}{bx}{n}

\usepackage{hyperref}
\hypersetup{
  colorlinks=true,
  linkcolor=blue,
  citecolor=teal,
  urlcolor=blue
}
\usepackage{xurl}
\usepackage{amsmath,amssymb}
\usepackage{amsthm}
\usepackage{booktabs}
\usepackage{tabularx}
\usepackage{graphicx}
\usepackage{needspace}
\usepackage{placeins}
\usepackage{float}
\usepackage{algorithm}
\usepackage{algpseudocode}

\theoremstyle{plain}
\newtheorem{definition}{Definition}
\newtheorem{proposition}{Proposition}
\newtheorem{corollary}{Corollary}

\title{Make Sparse Rewards Count: Density-Aware Reward Aggregation for Multi-Reward RL}

\author{
\textbf{Tong Zheng}\textsuperscript{1}\thanks{Equal contribution.}
\quad
\textbf{Skylar Zhai}\textsuperscript{2}\footnotemark[1]
\quad
\textbf{Zhan Cheng}\textsuperscript{3}\footnotemark[1]
\quad
\textbf{TianMing Sha}\textsuperscript{4}
\\
\textbf{Youling Huang}\textsuperscript{5}
\quad
\textbf{Shuo Zhou}\textsuperscript{6}
\quad
\textbf{Shaotong Qi}\textsuperscript{7}
\quad
\textbf{Jingcheng Liang}\textsuperscript{2}
\\
\textbf{Xuwei Ding}\textsuperscript{3}
\quad
\textbf{Pengcheng Xu}\textsuperscript{8}\thanks{Corresponding author}
\\[0.7em]
\normalfont
\textsuperscript{1}University of Chinese Academy of Sciences
\quad
\textsuperscript{2}University of Minnesota Twin Cities
\\
\textsuperscript{3}University of Wisconsin--Madison
\quad
\textsuperscript{4}Stony Brook University
\\
\textsuperscript{5}Dalian University of Technology
\quad
\textsuperscript{6}Beijing Foreign Studies University
\\
\textsuperscript{7}Southeast University
\quad
\textsuperscript{8}Kuaishou Technology
}

\iclrfinalcopy % Uncomment for camera-ready version, but NOT for submission.
\begin{document}

\maketitle
\lhead{Preprint}
\begin{abstract}
Multi-reward reinforcement learning trains large language models to satisfy multiple behavioral objectives simultaneously. Reward-wise normalization, as used in GDPO, preserves reward-specific relative information within rollout groups, but different objectives can still exhibit uneven learning progress. We study this behavior through \emph{advantage energy}, the sum of a reward's squared advantages over a batch. Under idealized GDPO normalization, we show that this energy is proportional to \emph{active-group density}: the fraction of rollout groups in which the reward provides nonzero relative advantages. This reveals a residual batch-level signal imbalance and provides a basis for calibrating reward contributions. Based on this relation, we propose \textbf{Density-Aware Reward Aggregation (DARA)}. We derive an inverse-square-root density correction that gives greater weight to signals from less frequently active rewards. DARA computes its weights from each rollout batch, adapting to changes in reward activity throughout training without modifying the underlying policy optimization objective. Experiments on tool calling and mathematical reasoning show that DARA learns the targeted behaviors faster than GDPO, reaching high format compliance in up to 26\% fewer training steps on tool calling and near-saturated length compliance in up to 65\% fewer steps on mathematical reasoning, while remaining competitive in final performance.
Our code is available at
\href{https://github.com/zhaihaotian/DARA}
{\texttt{github.com/zhaihaotian/DARA}}.

\end{abstract}

\section{Introduction}
\suppressfloats[t]
%Large language models (LLMs) are increasingly expected to satisfy a broad range of task requirements and human preferences during post-training \citep{ouyang2022training, touvron2023llama, rafailov2023direct}. Reinforcement learning (RL) has become a key approachach for this purpose, enabling models to optimize directly against task-specific and preference-based feedback \citep{nakano2021webgpt,ouyang2022training, guo2025deepseek}. As these requirements become more diverse, however, a single reward is often insufficient to characterize desirable model behavior. Modern LLMs may need to simultaneously achieve high correctness, maintain efficiency and safety, and interact reliably with external tools \citep{aggarwal2025l1, dai2024safe, schick2023toolformer}. This has naturally motivated multi-reward RL, where different behavioral requirements are represented as separate reward signals and optimized jointly \citep{dai2024safe, gdpo}. Within this setting, Group Relative Policy Optimization (GRPO) and its variants have become a widely adopted foundation, as they estimate relative advantages directly from groups of sampled responses without requiring a separate requiring a separate value model \citep{grpo, yu2026dapo}.

Reinforcement learning (RL) is widely used in large language model (LLM) post-training to optimize behavior from task-specific and preference-based feedback \citep{ouyang2022training,guo2025deepseek}. Group Relative Policy Optimization (GRPO) and its variants estimate relative advantages from groups of sampled responses without a separate value model \citep{grpo,yu2026dapo}. As task requirements become more diverse, however, a single reward is often insufficient to characterize desirable model behavior. Models may need to produce correct answers, satisfy length and format constraints, maintain safety, and use external tools reliably \citep{aggarwal2025l1,dai2024safe,schick2023toolformer}. These requirements motivate extending GRPO to multi-reward settings, where different behavioral requirements are represented as separate reward signals and optimized jointly \citep{dai2024safe,gdpo}.

%Extending GRPO to multiple rewards fundamentally changes how the policy optimization update is constructed. Standard GRPO first collapses rewards into a scalar score and then performs group-wise normalization, which can make distinct reward profiles indistinguishable \citep{grpo}. GDPO addresses this issue by normalizing each reward dimension independently before aggregation, thereby preserving reward-specific relative information \citep{gdpo}. Building on this formulation, recent methods further adapt reward-wise contributions according to signal variability, learning-stage informativeness, or consistency among advantages \citep{dvao,saw,gd2po}. However, these approaches mainly determine how strongly a reward should contribute once it becomes informative. They largely overlook how often each reward provides an informative signal in the first place. In practice, dense rewards such as correctness may vary across most rollout groups, while sparse rewards such as formatting or tool-use success activate only occasionally. Although GDPO normalizes their active advantages to comparable scales, sparse rewards still participate in far fewer effective updates, leading to slower and more seed-sensitive learning. 

In multi-reward GRPO, aggregating rewards before group-wise normalization can erase distinctions between reward combinations \citep{grpo}. GDPO addresses this issue by normalizing each reward dimension independently before aggregation, preserving reward-specific relative information \citep{gdpo}. However its mathematical reasoning and tool-calling experiments still show uneven learning across objectives. GDPO links this imbalance to the model's tendency to prioritize easier objectives, and reports that modest weight adjustments do not reliably change this priority \citep{gdpo}. These observations raise two questions: \textbf{\emph{What optimization mechanism underlies these differences in learning progress? Can reward weights be derived from this mechanism rather than relying on repeated empirical tuning?}}

%To this end, we introduce \textbf{Density-Aware Reward Aggregation (DARA)}. Starting from the role of advantages as coefficients in policy updates, our analysis relates active-group density to \emph{advantage energy}. From this relation, we derive a density-based calibration that gives the available comparisons from sparsely active rewards greater weight in the joint update. As shown in Figure~\ref{fig:tool_format_density}, DARA brings forward the phase of rapid format acquisition and reaches high format compliance in fewer training steps. Our empirical results on tool-calling and mathematical reasoning show favorable task performance alongside faster and more stable acquisition of the targeted behaviors.

We examine these questions through \emph{advantage energy}, the sum of each reward's squared advantages over a batch. Under idealized GDPO normalization, every active group contributes the same energy, making the batch total proportional to \emph{active-group density}, the fraction of rollout groups providing nonzero relative advantages for that reward. This relation offers a signal-based explanation for uneven learning and a basis for deriving reward weights. Figure~\ref{fig:tool_format_density} provides empirical support: higher active-group density accompanies faster improvement in format reward, and our calibration brings both the density peak and the rapid-learning phase forward.

Based on this analysis, we introduce \textbf{Density-Aware Reward Aggregation (DARA)}, matching each reward's energy to that of the highest-density reward. Less frequently active rewards receive larger weights when they do provide useful comparisons. DARA computes weights per rollout batch, adapting to changes in reward activity during training. Our main contributions are:
\begin{itemize}
     \item \textbf{A quantitative analysis of reward contributions.}
     We establish the relationship between active-group density and advantage energy, identifying a source of batch-level signal imbalance that remains after reward-wise normalization.

     \item \textbf{A density-based aggregation method.}
     We derive an inverse-square-root density calibration rule and develop DARA, which adapts advantage weights to reward activity during training and strengthens the relative signals supplied by infrequently active rewards.

     \item \textbf{An empirical evaluation of learning speed and stability.}
     Experiments on tool calling and mathematical reasoning show that DARA reaches high format compliance and near-saturated length compliance, respectively, in up to 26\% and 65\% fewer training steps than GDPO, while remaining competitive in final performance.
 \end{itemize}

\begin{figure*}[t]
    \vspace{-15pt}
    \centering
    \includegraphics[width=\textwidth]
    {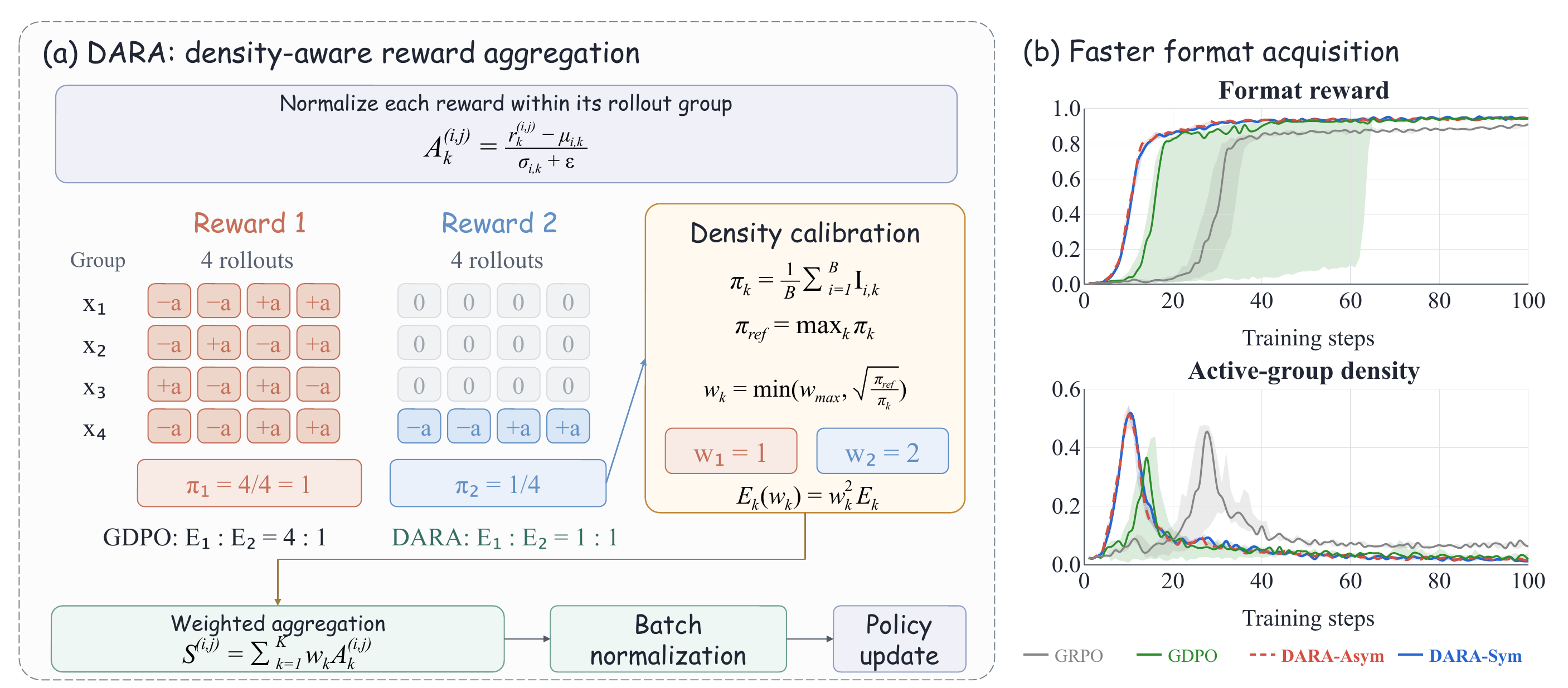}
    \vspace{-15pt}
    \caption{(a) DARA performs reward-wise group normalization, estimates active-group density, and applies density-based calibration before reward aggregation and batch-wise normalization. (b) Median format reward and active-group density, with shaded regions indicating the interquartile range. DARA reaches high format compliance in fewer training steps.}
    \label{fig:tool_format_density}
    \vspace{-8pt}
\end{figure*}

\FloatBarrier
\section{Related Work}

\paragraph{Reinforcement Learning for Large Language Models.}
Reinforcement learning has been widely used in LLM post-training to improve reasoning.
GRPO~\citep{grpo} estimates advantages through within-group reward comparisons without requiring a separate value model.
GSPO~\citep{gspo} introduces sequence-level importance ratios, while DAPO~\citep{yu2026dapo} improves training stability and efficiency and Dr.~GRPO~\citep{drgrpo} corrects length-related optimization bias.
For task-specific training, Search-R1~\citep{searchr1} enables search-augmented reasoning, while GiGPO~\citep{gigpo} improves credit assignment in long-horizon agent training.
Our method operates at the level of multi-reward advantage aggregation and is therefore orthogonal to both policy-update mechanisms and domain-specific task reward design.

\paragraph{Multi-Reward Reinforcement Learning.}
GDPO~\citep{gdpo} mitigates reward collapse in GRPO, where distinct reward combinations are mapped to identical advantages, through reward-wise group normalization.
A line of concurrent work builds on this formulation: DVAO~\citep{dvao} weights each reward by its within-group standard deviation, SAW~\citep{saw} by its batch-level coefficient of variation, and SA-MRPO~\citep{samrpo} by its remaining distance to the reward maximum.
GD$^2$PO~\citep{gd2po} filters rollouts whose reward-wise advantages have conflicting signs and reweights queries accordingly, while RVPO~\citep{rvpo} employs SoftMin aggregation to emphasize low-scoring reward dimensions.
SMOPD~\citep{smopd} combines multi-reward reinforcement learning with on-policy distillation by training and merging reward-specialized teachers.
In contrast, our aggregation is grounded in the relationship between active-group density and advantage energy, which identifies the signal imbalance left by reward-wise normalization and yields weights that equalize advantage energy across rewards.
Appendix~\ref{app:method_comparison} gives a detailed comparison of these methods.

% 如果更想强调理论是核心的话：In contrast, our adaptive reward aggregation is based on the theoretical relationship between active-group density and advantage energy, and adjusts dynamically throughout training.

\FloatBarrier
\section{Method}
\label{sec:method}

% DARA starts from a simple observation about multi-reward RL optimization.
% GDPO normalizes each reward separately within a rollout group, so different
% reward dimensions have comparable scales whenever they provide a useful
% relative comparison. However, rewards need not provide such comparisons
% equally often. Consider a correctness reward that varies within approximately $80\%$ of rollout groups and a format reward that
% varies within only $20\%$. Although their advantages have comparable
% scales within active groups, correctness contributes to roughly four
% times as many groups. In this section, we will first show that, under GDPO\citep{gdpo} normalization,
% the total advantage signal contributed by a reward is proportional to the
% fraction of groups in which that reward is active. This immediately leads to
% a density-based calibration rule.

We investigate the optimization mechanism behind uneven learning across objectives by examining reward-wise advantage signals at the batch level. We first show that, under idealized GDPO normalization, each reward’s advantage energy is proportional to its active-group density. We then derive a density-based weighting rule from this relation and use it to construct DARA.

%\begin{figure*}[t]
%    \vspace{-20pt}
%    \centering
%    \includegraphics[width=\textwidth,
%    trim=0 0.5cm 0 0,
%    clip]{figures/dara_batch_energy.pdf}
%    \vspace{-8pt}
%    \caption{Density calibration corrects the advantage-energy imbalance
%    accumulated across a batch.}
%    \label{fig:energy}
%    \vspace{-8pt}
%\end{figure*}

\subsection{Preliminaries}

Consider a batch $\mathcal{B}=\{x_i\}_{i=1}^{B}$ of $B$ prompts.
For each prompt $x_i$, the rollout policy samples $G$ responses
$\{y^{(i,j)}\}_{j=1}^{G}$.
Each response receives $K$ sequence-level rewards
$\{r_k^{(i,j)}\}_{k=1}^{K}$.
We use $\mu_{i,k}$ and $\sigma_{i,k}$ to denote the group mean and sample
standard deviation of reward $k$, with $\epsilon>0$. GRPO combines these rewards into a scalar reward and normalizes
the scalar within each rollout group:
\begin{equation}
r_{\mathrm{sum}}^{(i,j)}=
\sum_{k=1}^{K}r_k^{(i,j)},
\qquad
A_{\mathrm{GRPO}}^{(i,j)}=
\frac{
r_{\mathrm{sum}}^{(i,j)}-\mu_{i,\mathrm{sum}}
}{
\sigma_{i,\mathrm{sum}}+\epsilon
}.
\label{eq:grpo}
\end{equation}

GDPO reverses this order. It normalizes each reward dimension
independently,
\(
A_k^{(i,j)}=
\frac{
r_k^{(i,j)}-\mu_{i,k}
}{
\sigma_{i,k}+\epsilon
},
\)
and then aggregates the normalized advantages:
\begin{equation}
S_{\mathrm{GDPO}}^{(i,j)}=
\sum_{k=1}^{K}A_k^{(i,j)},
\qquad
\widehat A_{\mathrm{GDPO}}^{(i,j)}=
\frac{
S_{\mathrm{GDPO}}^{(i,j)}
-\mu_{\mathcal B}(S_{\mathrm{GDPO}})
}{
\sigma_{\mathcal B}(S_{\mathrm{GDPO}})+\epsilon
}.
\label{eq:gdpo}
\end{equation}
This reward-wise normalization preserves contrasts from individual reward
dimensions that may disappear after early scalarization. It also ensures
that active reward dimensions contribute at comparable scales within a
group. However, from the batch perspective, different rewards may contribute to different fractions of rollout groups depending on how often they are active. We next quantify how this difference translates into their overall contribution to the policy update.

\subsection{From Active Groups to Advantage Energy}
\label{sec:density_energy}
A reward provides a group-relative learning signal only when its values differ
across responses in the rollout group. We therefore call reward $k$ \emph{active} in group $i$ when it induces a
non-zero group-relative advantage. Active-group density is the fraction of rollout groups in which this
occurs
\begin{equation}
I_{i,k}=
\mathbf{1}\left[
\sum_{j=1}^{G}
\left|A_k^{(i,j)}\right|>\tau
\right],\qquad \pi_k=
\frac{1}{B}
\sum_{i=1}^{B} I_{i,k}.
\label{eq:group_activity}
\end{equation}

Thus, $B\pi_k$ counts how many groups provide reward $k$ to influence
the policy update. For a binary reward with rollout success probability $p_k$, a group is active
whenever it contains both successful and unsuccessful responses. Differences in per-rollout success rates can therefore translate into much larger differences in active-group density. This reveals why density can vary substantially across rewards.
A very difficult reward produces mostly all-failure groups, while a nearly
saturated reward produces mostly all-success groups. Rewards in an intermediate regime are active much more often. We now quantify how this difference affects optimization.

% \newpage
\begin{definition}[Advantage Energy]
\label{def:advantage_energy}
The advantage energy of reward $k$ over a batch is
\begin{equation}
E_k
\triangleq
\sum_{i=1}^{B}
\sum_{j=1}^{G}
\left(A_k^{(i,j)}\right)^2.
\label{eq:advantage_energy}
\end{equation}
\end{definition}

$E_k$ measures how much reward-wise advantage coefficient mass is available
to drive the policy update. To see this connection, let $g^{(i,j)}$ denote the
policy score of response $y^{(i,j)}$. The contribution associated with reward
$k$ takes the form
\(
u_k^{(i,j)}=
A_k^{(i,j)}g^{(i,j)},
\) which motivates
using $E_k$ to compare the reward-wise signal supplied by different reward
dimensions. The key property of idealized GDPO normalization is that every active
group contributes the same amount of advantage energy. Specifically,
consider $\epsilon=\tau=0$, assign zero advantages to constant-reward
groups, and use the sample standard deviation with divisor $G-1$.
Each active group then contributes $G-1$, giving
\begin{equation}
E_k=B\pi_k(G-1).
\label{eq:density_energy}
\end{equation}
This identity exposes the residual imbalance left by reward-wise
normalization. GDPO equalizes the contribution of each active group under
these assumptions, while the total energy accumulated across the batch
still scales linearly with active-group density. Appendix~\ref{app:energy_gradient} shows that, under an explicit covariance condition, this energy also governs the fluctuation of each reward's policy-gradient contribution. Figure~\ref{fig:bfcl_v4_checkpoint_dynamics} illustrates uneven learning across objectives: tool-calling accuracy is acquired earlier than format compliance. Figure~\ref{fig:tool_format_density} further shows the relationship between active-group density and the learning dynamics of the format reward. Its rapid increase coincides with a density peak, and density declines as the reward approaches its plateau. Compared with GRPO and GDPO, DARA brings both the density peak and the rapid-learning phase forward. Together with Equation~\ref{eq:density_energy}, these observations motivate calibrating reward-wise advantage signals.

\subsection{Density-Aware Reward Aggregation}
\label{sec:dara}
Equation~\ref{eq:density_energy} suggests a direct correction.
Suppose reward $k$ is multiplied by a channel weight $\alpha_k$ before
aggregation. Its advantage energy becomes
$E_k(\alpha_k)=\alpha_k^2 E_k$.
Let $\pi_{\mathrm{ref}}=\max_{1\leq k\leq K}\pi_k$ denote the largest
active-group density in the batch. For $\pi_k>0$, matching the reference
energy gives $\alpha_k^\star=\sqrt{\pi_{\mathrm{ref}}/\pi_k}$.
Under the assumptions of Equation~\ref{eq:density_energy}, this yields
$E_k(\alpha_k^\star)=B\pi_{\mathrm{ref}}(G-1)$, independently of the
original density $\pi_k$.
This gives the central rule behind DARA: rewards that become active less
frequently receive a larger coefficient when they do provide a useful
comparison. We use a capped correction in practice:
\begin{equation}
w_k=
\begin{cases}
\displaystyle
\min\!\left(w_{\max},\sqrt{\frac{\pi_{\mathrm{ref}}}{\pi_k}}\right),
& \pi_k>0,\\[4pt]
1, & \pi_k=0,
\end{cases}
\label{eq:density_weight}
\end{equation}
where $w_{\max}\geq 1$ limits amplification when a reward is active in very
few groups. A reward with zero measured density receives no amplification.
The cap can prevent exact energy matching when the required correction
exceeds $w_{\max}$. Applying $w_k$ to both positive and negative advantages gives
\textbf{DARA-Sym}, the direct implementation of the energy calibration
above. 

However, increasing both signs can also strengthen negative
coefficients when reward dimensions conflict. For example, in a rollout
group containing both length-compliant and overlong responses, a correct
but overlong response can receive a negative length advantage. Symmetric
calibration amplifies this negative coefficient with the positive
coefficients of compliant responses. To reinforce favorable relative comparisons, we introduce
\textbf{DARA-Asym}. It uses the same density weights but applies the
additional amplification only to positive advantages. Both variants can
be written as
\begin{equation}
\widetilde A_{k,v}^{(i,j)}=
\begin{cases}
w_k A_k^{(i,j)}, & v=\mathrm{sym},\\[3pt]
A_k^{(i,j)}+(w_k-1)\bigl[A_k^{(i,j)}\bigr]_+,
& v=\mathrm{asym},
\end{cases}
\label{eq:dara_variants}
\end{equation}
where $[a]_+=\max(a,0)$.
DARA-Asym retains negative coefficients at their original GDPO scale
before aggregation, while strengthening positive coefficients from
infrequently active rewards. This choice limits the additional negative
pressure introduced by density calibration. Its motivation follows from
the density imbalance, while the exact energy-matching result applies to
uncapped DARA-Sym under the stated assumptions.
We use DARA-Asym as our main method and retain DARA-Sym to evaluate the
effect of symmetric calibration. Finally, for each rollout batch, we compute the reward-wise advantages,
estimate their active-group densities, and obtain the weights.
\begin{equation}
S_v^{(i,j)}=
\sum_{k=1}^{K}\widetilde A_{k,v}^{(i,j)},
\qquad
\widehat A_v^{(i,j)}=
\frac{
S_v^{(i,j)}-\mu_{\mathcal B}(S_v)
}{
\sigma_{\mathcal B}(S_v)+\epsilon
},
\quad v\in\{\mathrm{sym},\mathrm{asym}\}.
\label{eq:dara}
\end{equation}
We then apply the selected calibration in
Equation~\ref{eq:dara_variants}, aggregate the calibrated advantages,
and normalize them over the batch. The resulting sequence-level advantage is assigned to the generated tokens
and used in the same clipped policy optimization objective as GDPO.

\FloatBarrier
\section{Experiments}

\subsection{Tool-Calling}

\paragraph{Task and training.}
Following ToolRL \citep{toolrl} and GDPO \citep{gdpo}, we train models to select
tools and generate their arguments from a user query and tool descriptions.
We use the training data provided by ToolRL, with
3,920 training examples and 80 held-out validation examples. Responses follow
the prescribed structure using \texttt{<think>}, \texttt{<tool\_call>}, and
\texttt{<response>} blocks as appropriate. Training uses two rewards: a binary
format reward $\{0,1\}$ checks the required structure, while a correctness reward in
$[-3,3]$ assigns credit for matching tool names, parameter names, and parameter
values against the reference calls. We compare both DARA variants, DARA-Asym and DARA-Sym (Section~\ref{sec:dara}), with GRPO~\citep{grpo}, GDPO~\citep{gdpo}, and the concurrent methods DVAO~\citep{dvao} and GD$^2$PO-Hard~\citep{gd2po}, which we reimplement under the same training configuration, using
Qwen2.5-1.5B-Instruct and Qwen2.5-3B-Instruct. Full optimization and rollout settings are given in
Appendix~\ref{app:tool_calling_details}.

\paragraph{Evaluation.}
We evaluate the final checkpoints on BFCL-v4 \citep{bfcl}, covering Live,
Non-Live, and Multi-Turn tool-calling tasks. Accuracy measures whether the
model generates the correct functions and arguments, while Format measures
adherence to the required output structure.

\begin{table}[H]
\vspace{-10pt}
\caption{BFCL-v4 performance at step 100. Acc.\ denotes tool-calling accuracy
and Format denotes output-format compliance. Average is the mean over
Live, Non-Live, and Multi-Turn.}
\label{tab:bfcl_final}
\centering
\fontsize{8}{9.2}\selectfont
\setlength{\tabcolsep}{3pt}
\renewcommand{\arraystretch}{1.10}
\begin{tabular*}{\linewidth}{@{\extracolsep{\fill}}cl*{8}{c}@{}}
\toprule
& & \multicolumn{2}{c}{Live}
& \multicolumn{2}{c}{Non-Live}
& \multicolumn{2}{c}{Multi-Turn}
& \multicolumn{2}{c}{Average} \\
\cmidrule(lr){3-4}
\cmidrule(lr){5-6}
\cmidrule(lr){7-8}
\cmidrule(l){9-10}
Size & Method
& Acc. & Format
& Acc. & Format
& Acc. & Format
& Acc. & Format \\
\midrule
1.5B & Base
& 9.33 & 0.37
& 12.06 & 0.12
& 0.38 & 0.21
& 7.25 & 0.23 \\

& GRPO
& 64.80 & 99.88
& 78.04 & 99.91
& 2.55 & 43.24
& 48.46 & 81.01 \\

& GDPO
& 67.05 & 99.94
& 79.30 & 99.97
& 4.75 & 91.40
& 50.37 & 97.11 \\

& DVAO
& 64.99 & 99.75
& 76.84 & 99.62
& 4.08 & 81.19
& 48.64 & 93.52 \\

& GD$^2$PO-Hard
& 67.06 & \textbf{99.97}
& 78.96 & \textbf{100.00}
& 4.38 & 86.59
& 50.13 & 95.52 \\
\midrule
& DARA-Asym
& 68.10 & 99.67
& 79.18 & 99.87
& 4.50 & \textbf{94.17}
& 50.59 & \textbf{97.90} \\

& DARA-Sym
& \textbf{68.42} & 99.91
& \textbf{80.25} & 99.97
& \textbf{4.85} & 92.77
& \textbf{51.17} & 97.55 \\
\midrule
3B & Base
& 47.67 & 54.92
& 36.23 & 34.23
& 1.75 & 9.16
& 28.55 & 32.77 \\

& GRPO
& 69.31 & 99.78
& 81.86 & 99.84
& 6.53 & 66.66
& 52.57 & 88.76 \\

& GDPO
& 69.30 & 99.93
& 82.60 & \textbf{100.00}
& 9.07 & 89.79
& 53.66 & 96.57 \\

& DVAO
& 70.00 & 99.90
& 83.03 & 99.99
& \textbf{9.62} & 91.21
& \textbf{54.22} & 97.04 \\

& GD$^2$PO-Hard
& 69.45 & 99.73
& \textbf{83.11} & \textbf{100.00}
& 9.12 & 91.28
& 53.90 & 97.00 \\
\midrule
& DARA-Asym
& 69.50 & 99.94
& 82.23 & 99.98
& 9.16 & 89.57
& 53.63 & 96.50 \\

& DARA-Sym
& \textbf{70.39} & \textbf{99.97}
& 83.05 & \textbf{100.00}
& 9.00 & \textbf{91.69}
& 54.15 & \textbf{97.22} \\
\bottomrule
\end{tabular*}
\vspace{-10pt}
\end{table}

\begin{figure}[t]
    \vspace{-15pt}
    \centering
    \includegraphics[width=\textwidth,
    trim=0 0.28cm 0 0.38cm,
    clip]{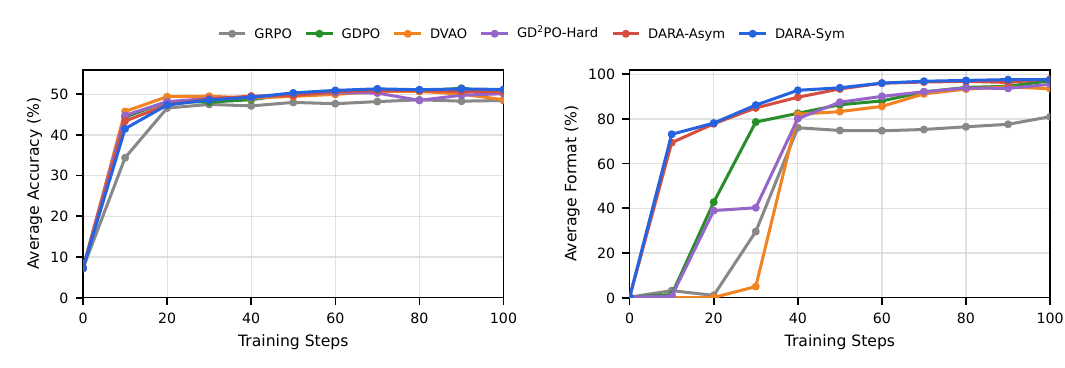}
    \vspace{-14pt}
    \caption{Checkpoint-wise BFCL-v4 evaluation of Qwen2.5-1.5B-Instruct. Left: Average Accuracy; right: Average Format, each averaged over Live, Non-Live, and Multi-Turn.}
    \label{fig:bfcl_v4_checkpoint_dynamics}
    \vspace{-10pt}
\end{figure}

\paragraph{Faster convergence and downstream generalization.}
Figure~\ref{fig:tool_format_density} shows that both DARA variants acquire the
format objective substantially faster, with the median format reward reaching
0.8 at steps 14--15 versus 19 for GDPO and 34 for GRPO. This faster
convergence transfers to downstream tool-calling performance
(Figure~\ref{fig:bfcl_v4_checkpoint_dynamics}). At step 60, DARA-Sym and
DARA-Asym achieve the highest Average Accuracy (50.94\% and 50.69\%) and
Average Format (96.06\% and 96.02\%) among all compared methods, whereas the
baselines reach at most 50.50\% Average Accuracy and 90.11\% Average Format.
Table~\ref{tab:bfcl_final} reports the final checkpoints at step 100, when most
methods have largely converged. DARA remains competitive at this stage: both
variants improve Average Accuracy and Average Format over GRPO and GDPO on the
1.5B model. On the 3B model, DARA-Sym achieves the highest Average Format and
near-best Average Accuracy, while DARA-Asym performs on par with GDPO.
DARA-Asym converges more consistently across seeds, whereas DARA-Sym attains
higher final Average Accuracy at both scales.

%\FloatBarrier

\paragraph{Convergence aligns with active-group density.}
Our analysis yields a testable prediction: for a binary reward with success probability $p$, a rollout group is active with probability $1-p^G-(1-p)^G$, which grows with the group size $G$ for any $p$, so a larger $G$ should raise active-group density and accelerate learning, even at a fixed response budget (Appendix~\ref{app:heterogeneous_density}). We vary $G\in\{4,8,16,32\}$ with 2,048 responses per step (Figure~\ref{fig:group_size_format_convergence}), and the results match this prediction: larger $G$ raises the format active-group density for all methods, and the format reward reaches $0.8$ earlier. GDPO learns the format objective slowly and inconsistently at small $G$ but reliably at $G=16$ and $32$. DARA also benefits from larger $G$, yet reaches $0.8$ earlier than GRPO and GDPO at every group size, as its density calibration compensates for sparse reward activity.

\begin{figure}[t]
    \centering
    \includegraphics[width=\textwidth]{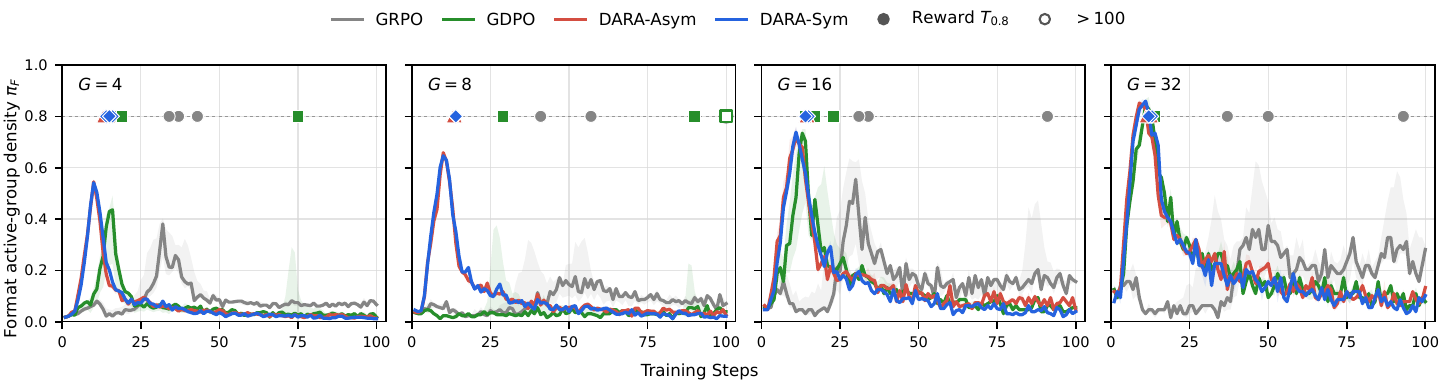}
    \vspace{-14pt}
    \caption{Format active-group density across rollout-group sizes.
    Curves show the median and shaded regions the interquartile range.
    Markers on the dashed line show when format reward first reaches 0.8;
    open markers indicate that it does not reach 0.8 within 100 steps.}
    \label{fig:group_size_format_convergence}
    \vspace{-10pt}
\end{figure}

\paragraph{Extension to three rewards.}
We further add a third reward that limits the reasoning length: it equals one when the \texttt{<think>} block contains at most 16
words and zero otherwise. This setting tests whether a new objective competes
with existing rewards for learning signal. At the final checkpoints, adding the length reward to GDPO
lowers Average Format from 97.11\% (Table~\ref{tab:bfcl_final}) to 94.39\%
(Table~\ref{tab:tool_three_rewards}), and Multi-Turn format from 91.40\% to
83.17\%. DARA largely avoids this interference: Average Format changes from
97.90\% to 96.93\% for DARA-Asym and from 97.55\% to 97.43\% for DARA-Sym,
while both variants achieve higher Average Len.\ than GDPO and Average
Accuracy is comparable across methods. By calibrating each reward
by its own active-group density, DARA incorporates the new objective
without crowding out the learning signal of existing rewards.

\begin{table}[b]
\vspace{-12pt}
\caption{BFCL-v4 performance at step 100 of Qwen2.5-1.5B-Instruct trained with
correctness, format, and length rewards. Len.\ denotes compliance; the other columns follow Table~\ref{tab:bfcl_final}.}
\label{tab:tool_three_rewards}
\centering
\fontsize{8}{9.2}\selectfont
\setlength{\tabcolsep}{1.6pt}
\renewcommand{\arraystretch}{1.00}
\begin{tabular*}{\linewidth}{@{\extracolsep{\fill}}cl*{12}{c}@{}}
\toprule
& & \multicolumn{3}{c}{Live}
& \multicolumn{3}{c}{Non-Live}
& \multicolumn{3}{c}{Multi-Turn}
& \multicolumn{3}{c}{Average} \\
\cmidrule(lr){3-5}
\cmidrule(lr){6-8}
\cmidrule(lr){9-11}
\cmidrule(l){12-14}
Size & Method
& Acc. & Format & Len.
& Acc. & Format & Len.
& Acc. & Format & Len.
& Acc. & Format & Len. \\
\midrule
1.5B & Base
& 9.33 & 0.37 & 1.33
& 12.06 & 0.12 & 0.73
& 0.38 & 0.21 & 3.58
& 7.25 & 0.23 & 1.88 \\

& GRPO
& 67.80 & \textbf{100.00} & 86.34
& 78.26 & 99.96 & 78.62
& \textbf{4.06} & 64.14 & 72.59
& 50.04 & 88.03 & 79.19 \\

& GDPO
& 66.62 & \textbf{100.00} & \textbf{100.00}
& \textbf{80.94} & \textbf{100.00} & 99.93
& 3.81 & 83.17 & 93.97
& 50.46 & 94.39 & 97.96 \\
\midrule
& DARA-Asym
& \textbf{68.02} & \textbf{100.00} & 99.41
& 80.00 & \textbf{100.00} & \textbf{100.00}
& 3.88 & 90.80 & 94.74
& \textbf{50.63} & 96.93 & 98.05 \\

& DARA-Sym
& 67.95 & \textbf{100.00} & \textbf{100.00}
& 80.21 & \textbf{100.00} & 99.99
& 3.31 & \textbf{92.30} & \textbf{96.16}
& 50.49 & \textbf{97.43} & \textbf{98.72} \\
\bottomrule
\end{tabular*}
\end{table}
% % The corresponding training-dynamics figure will use two rows (1.5B/3B)
% % and three columns (correctness/format/length).

\subsection{Mathematical reasoning}
\paragraph{Task and training.}
We study mathematical reasoning under two competing objectives: answer correctness and response-length compliance. We evaluate DARA on DeepSeek-R1-1.5B, Qwen3-4B-Instruct~\citep{yang2025qwen3}, and DeepSeek-R1-7B~\citep{guo2025deepseek}, and compare against GRPO, GDPO, DVAO~\citep{dvao}, and GD$^2$PO-Hard~\citep{gd2po}, adapting them to the same training configuration as GDPO. This includes DAPO-style \citep{yu2026dapo} training, DeepScaleR-Preview dataset \citep{luo2025deepscaler}, and DeepSeek-R1 prompt format \citep{guo2025deepseek}. All methods train with a maximum response length of 8000 tokens and with two binary rewards: a correctness reward in $\{0,1\}$ indicating whether the extracted final answer matches the ground truth, and a length reward in $\{0,1\}$ indicating whether the response contains at most 4,000 tokens. 

\paragraph{Evaluation.} We evaluate on MATH-500 \citep{hendrycks2021measuring}, AIME 2024 \citep{maa2024aime}, AMC 2022/2023 \citep{maaAMC}, Minerva \citep{lewkowycz2022solving}, and OlympiadBench \citep{he2024olympiadbench}. For each method, we report pass@1 accuracy and the fraction of responses exceeding 4,000 tokens (Exceed). To study optimization speed directly, we evaluate checkpoints every 10 training steps during the first 100 steps.

\paragraph{Training dynamics.}
We first examine how the two reward dimensions evolve during optimization.
Figure~\ref{fig:math_training_dynamics} shows the correctness and length
rewards over the first 100 training steps on DeepSeek-R1-1.5B.
The length objective is acquired considerably faster than the
correctness objective, with most methods rapidly moving toward high length
compliance during early training. Both DARA variants reach high length compliance early,
while DARA-Asym generally retains a higher correctness reward than
DARA-Sym. The difference between methods becomes more visible as the
length reward approaches saturation. Early in training, many rollout
groups contain both compliant and over-length responses, so the length
reward is informative for a large fraction of groups. As compliance
improves, such mixed groups become increasingly rare. The available
relative comparisons therefore become sparser precisely in the regime
where further improvement becomes difficult. This is the regime targeted
by DARA's density calibration.

\begin{figure*}[t]
    \centering
    \includegraphics[width=0.98\textwidth]{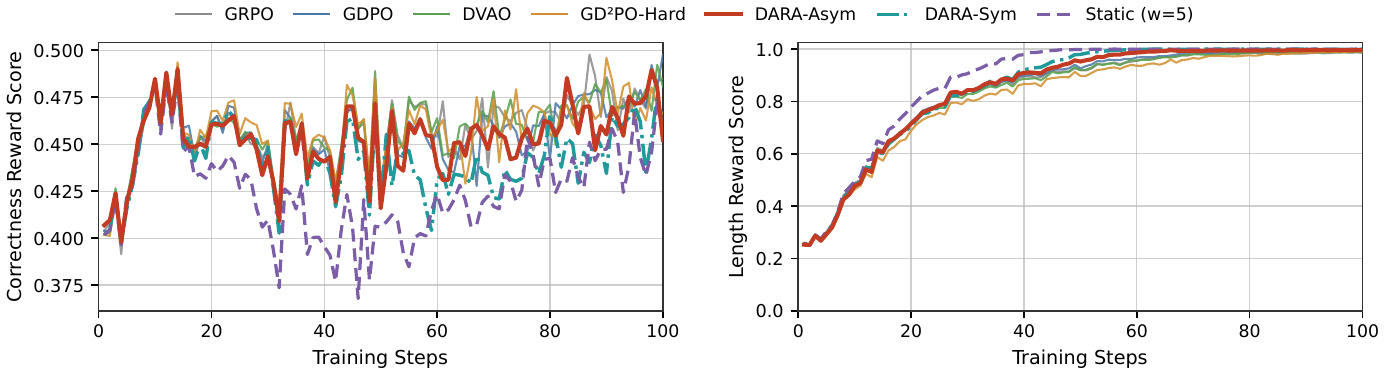}
    \vspace{-4pt}
    \caption{
    Training reward dynamics for mathematical reasoning on DeepSeek-R1-1.5B.
    Left: correctness reward; right: length reward. Static denotes
    the Static-Sym control in Table~\ref{tab:math_ablations}(a).
    }
    \label{fig:math_training_dynamics}
    \vspace{-10pt}
\end{figure*}

\paragraph{The advantage grows near reward saturation.}
To quantify the training dynamics more directly, we measure the first stable
crossing of different length-compliance thresholds.
Figure~\ref{fig:math_thresholds} reports the optimization step at which each
method stably reaches a target compliance level. A stable crossing requires
the trailing 10-step mean to remain above the target for the following
20 optimization steps. At moderate compliance levels, the
differences between methods are relatively small. For example, on
Qwen3-4B-Instruct, DARA-Asym, DARA-Sym, and GDPO all reach the 80\%
threshold at step~20. The separation becomes substantially larger at
stricter thresholds: the same methods reach 99\% compliance at steps~66,
59, and~111, respectively, and at steps~56, 50, and~142 on DeepSeek-R1-7B,
corresponding to 41--65\% fewer steps than GDPO. At lower compliance, the length reward varies
within many rollout groups and therefore already supplies frequent
relative comparisons. Near saturation, most sampled responses satisfy
the length constraint, causing the active-group density of the length
reward to decrease. DARA assigns larger weight to the remaining
informative comparisons, so its relative benefit becomes more pronounced
as the objective becomes sparse.

\begin{figure*}[t]
    %\vspace{-5pt}
    \centering
    \includegraphics[width=0.48\textwidth]{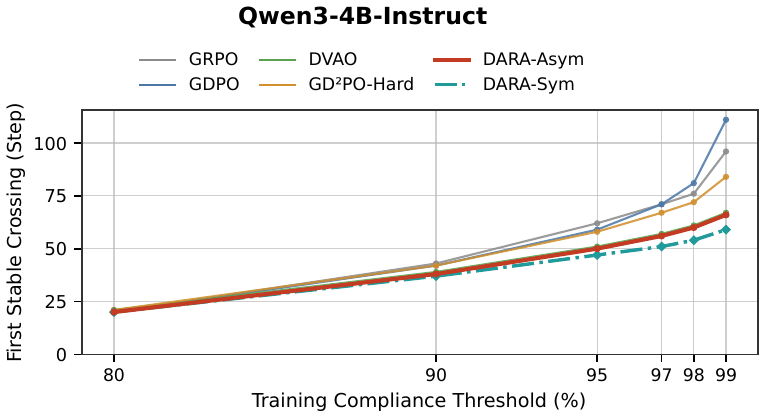}
    \hfill
    \includegraphics[width=0.48\textwidth]{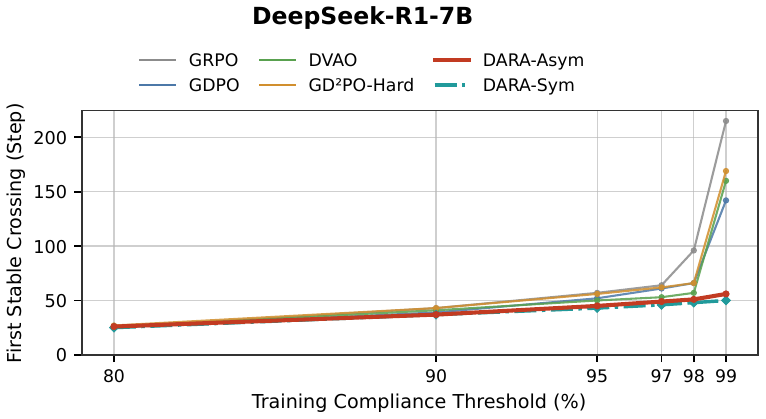}
    \vspace{-4pt}
    \caption{
    First stable crossing step for increasingly strict training-side
    length-compliance thresholds on Qwen3-4B-Instruct (left) and
    DeepSeek-R1-7B (right). Lower values indicate faster acquisition.
    }
    \label{fig:math_thresholds}
    \vspace{-5pt}
\end{figure*}

\paragraph{Training acceleration transfers to held-out evaluation.}
We next examine whether the faster acquisition observed in rollout rewards
transfers to external mathematical reasoning benchmarks.
Figure~\ref{fig:math_eval_trajectories} evaluates DeepSeek-R1-7B every 10
training steps and reports both macro pass@1 accuracy and the percentage of
responses exceeding 4,000 tokens. The Exceed rate falls rapidly during the first
half of training, and both DARA variants
enter the low-Exceed regime earlier than GDPO. Both DARA variants first reach
95\% held-out length compliance at step~50, compared with step~70 for
GDPO. The accuracy trajectories reveal
a common optimization pattern across methods: enforcing the length
constraint initially reduces macro accuracy, followed by partial recovery
later in training. Both DARA variants follow the same overall trajectory.
DARA-Sym achieves lower Exceed rates, while DARA-Asym retains higher
accuracy, revealing different accuracy--compliance trade-offs under
the two implementations.

\begin{figure*}[t]
    \centering
    \includegraphics[width=0.98\textwidth]{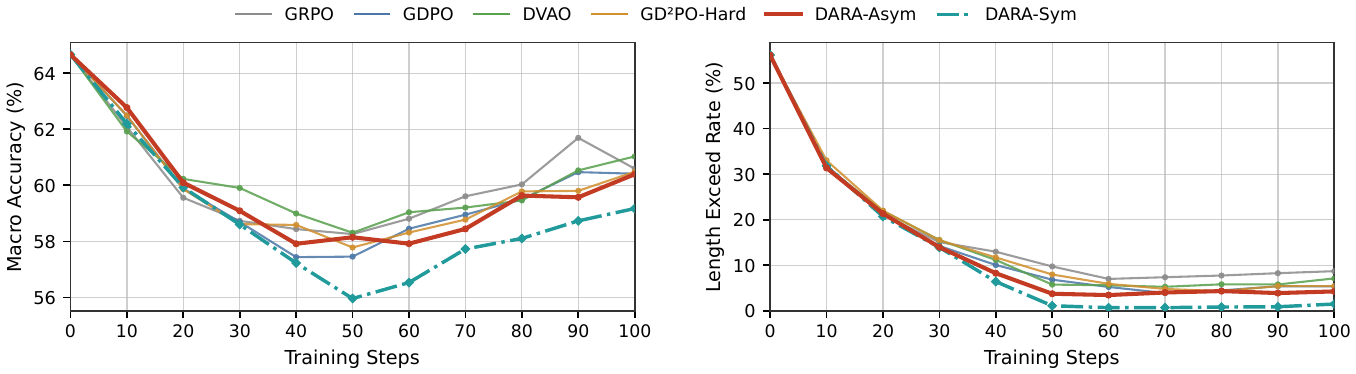}
    \vspace{-4pt}
    \caption{
    Checkpoint-wise held-out mathematical reasoning evaluation on
    DeepSeek-R1-7B. Left: macro pass@1 accuracy; right: percentage of
    responses exceeding 4,000 tokens.
    }
    \label{fig:math_eval_trajectories}
    \vspace{-10pt}
\end{figure*}

%\paragraph{Faster attainment of the held-out compliance target.} The full trajectories suggest that DARA reaches useful levels of length compliance earlier. We summarize this comparison in Figure~\ref{fig:math_eval_speed} by converting Exceed into length compliance and marking a common 95\% target. DARA reaches the target at the earliest or tied-earliest measured checkpoint for all three model scales. On DeepSeek-R1-1.5B, DARA reaches 95\% held-out compliance at step 60. On Qwen3-4B-Instruct and DeepSeek-R1-7B, it reaches the same target at step 50. GDPO reaches 95\% compliance at step 70 in all three settings. On the 4B model, DVAO also reaches the target at step 50, so DARA is tied for the earliest measured checkpoint in this setting. The agreement between the training-side threshold analysis and held-out evaluation is important. DARA first accelerates the acquisition of the length reward during policy optimization, with the largest separation appearing near high compliance. The same acceleration is subsequently visible on unseen reasoning problems, indicating that the effect extends beyond the rollout batches used to construct the policy update.
%\begin{figure*}[t]
%    \vspace{-5pt}
%    \centering
%    \includegraphics[width=\textwidth,
%    trim=0 0.1cm 0 0,
%    clip]{figures/eval_target_attainment.pdf}
%    \vspace{-8pt}
%    \caption{Held-out length compliance during the first 100 optimization steps.}
%    \label{fig:math_eval_speed}
%    \vspace{-8pt}
%\end{figure*}

\begin{table*}[t]
\vspace{-8pt}
\caption{
Aggregate mathematical reasoning results at step 50.
Acc is the macro-average pass@1 accuracy over the five benchmarks,
Exceed is the macro-average percentage of responses longer than 4,000
tokens, and Joint is the macro-average percentage of responses that
are both correct and within the length limit.
Base denotes the pretrained checkpoint. All values are percentages.
}
\label{tab:math_aggregate}
\vspace{6pt}
\centering
\fontsize{8}{9.2}\selectfont
\setlength{\tabcolsep}{3pt}
\renewcommand{\arraystretch}{1.10}
\begin{tabular*}{\textwidth}{@{\extracolsep{\fill}}l*{9}{c}@{}}
\toprule
& \multicolumn{3}{c}{DeepSeek-R1-1.5B}
& \multicolumn{3}{c}{Qwen3-4B-Instruct}
& \multicolumn{3}{c}{DeepSeek-R1-7B} \\
\cmidrule(lr){2-4}
\cmidrule(lr){5-7}
\cmidrule(lr){8-10}
Method
& Acc $\uparrow$ & Exceed $\downarrow$ & Joint $\uparrow$
& Acc $\uparrow$ & Exceed $\downarrow$ & Joint $\uparrow$
& Acc $\uparrow$ & Exceed $\downarrow$ & Joint $\uparrow$ \\
\midrule
Base
& 48.93 & 65.23 & 27.75
& 68.15 & 33.60 & 51.69
& 64.66 & 56.16 & 36.32 \\
\midrule
GRPO
& 45.52 & 11.03 & 44.69
& \textbf{62.14} & 8.07 & \textbf{60.48}
& 58.26 & 9.73 & 56.74 \\

GDPO
& 44.12 & 8.66 & 43.47
& 60.82 & 7.64 & 59.19
& 57.45 & 6.82 & 56.48 \\

DVAO
& 45.58 & 11.35 & 44.57
& 61.37 & 4.77 & 60.34
& 58.15 & 5.75 & 57.53 \\

GD$^2$PO-Hard
& 44.24 & 11.41 & 43.45
& 60.58 & 7.28 & 59.16
& 57.78 & 7.96 & 56.35 \\
\midrule
DARA-Asym
& \textbf{45.83} & 8.16 & \textbf{45.14}
& 60.30 & 4.38 & 59.68
& \textbf{58.31} & 3.73 & \textbf{57.83} \\

DARA-Sym
& 44.43 & \textbf{5.18} & 44.12
& 58.76 & \textbf{1.69} & 58.52
& 55.96 & \textbf{1.09} & 55.79 \\
\bottomrule
\end{tabular*}
\vspace{-10pt}
\end{table*}

\paragraph{Accuracy--compliance trade-off.}
We compare the step-50 checkpoints in
Table~\ref{tab:math_aggregate}. All multi-reward methods substantially
reduce the rate of over-length generations relative to the initial models.
Among the methods in the table, DARA-Sym achieves the lowest average
Exceed rate on all three models, at 5.18\%, 1.69\%, and 1.09\% on
DeepSeek-R1-1.5B, Qwen3-4B-Instruct, and DeepSeek-R1-7B, respectively.
DARA-Asym retains higher accuracy than DARA-Sym on all three models,
reaching 45.83\%, 60.30\%, and 58.31\%, with corresponding Exceed rates
of 8.16\%, 4.38\%, and 3.73\%. Pass@1 accuracy is more
method-dependent. Different aggregation strategies lead on different
model scales, and the method with the lowest Exceed rate does not
necessarily achieve the highest accuracy. This motivates
considering the joint objective rather than either reward in isolation.
On both DeepSeek-R1 models, DARA-Asym achieves the highest Joint score
among the compared methods, reaching 45.14\% and 57.83\%, together with the
highest accuracy. These results characterize the benefit of DARA in terms of optimization dynamics. Its advantage over GDPO is
already visible in the training reward, becomes increasingly pronounced
as length compliance approaches saturation, and transfers to earlier
attainment of the same behavior on held-out benchmarks.

\paragraph{Fixed-weight ablation.}
A natural question is whether DARA's faster length convergence simply comes
from assigning a larger weight to the sparse length reward. We compare DARA with fixed-weight variants that set the correctness weight to one
and the length weight to five (Table~\ref{tab:math_ablations}a). Strong fixed weighting indeed improves length
compliance, but causes a clear accuracy loss. For DARA-Asym, fixing the length
weight reduces Exceed from 4.21\% to 3.14\%, while Acc drops from 48.03\% to
46.84\%. For DARA-Sym, Exceed changes only marginally from 0.89\% to 0.87\%,
while Acc decreases from 46.71\% to 45.97\%. These results show that simply
increasing the sparse-reward weight can accelerate length optimization, but
does so at a larger cost to accuracy.

\paragraph{Extension to a thinking-oriented model.}
At step 50, DARA-Asym performs strongly on both DeepSeek-R1 models, whereas
GRPO gives the best result on Qwen3-4B-Instruct. We examine whether this gap
persists after longer training and whether it depends on the model's reasoning
mode. We extend training to step 100, when the length objective has largely
converged across methods, on both Qwen3-4B-Instruct and Qwen3-4B-Thinking.
After convergence, the gap becomes much smaller for both models. On Qwen3-4B-Instruct, DARA-Asym achieves 62.17\% Acc versus 62.99\% for GRPO, while reducing Exceed
from 0.70\% to 0.25\%. On Qwen3-4B-Thinking, it achieves 64.23\% Acc versus
64.82\% for GRPO, while reducing Exceed from 1.86\% to 0.94\%. Thus, once
length compliance has largely converged, DARA-Asym retains accuracy comparable
to GRPO while enforcing the length constraint more effectively.

\begin{table*}[t]
\vspace{-6pt}
\caption{
Left: dynamic versus fixed reward weighting on
DeepSeek-R1-1.5B. Right: comparison between Qwen3-4B Instruct and Thinking;
underlined values mark DARA-Asym.
}
\label{tab:math_ablations}
\vspace{4pt}
\centering
\fontsize{8}{9.2}\selectfont

\begin{minipage}[t]{0.45\textwidth}
    \centering
    \vspace{16pt}
    {\bfseries (a) Fixed-weight ablation\par}
    \vspace{3pt}

    \setlength{\tabcolsep}{3.2pt}
    \renewcommand{\arraystretch}{1.10}
    \begin{tabular}{@{}lccc@{}}
        \toprule
        Method
        & Acc $\uparrow$
        & Exceed $\downarrow$
        & Joint $\uparrow$ \\
        \midrule
        DARA-Asym
        & \textbf{48.03}
        & 4.21
        & \textbf{47.76} \\

        Static-Asym ($w=5$)
        & 46.84
        & \textbf{3.14}
        & 46.72 \\
        \midrule
        DARA-Sym
        & \textbf{46.71}
        & 0.89
        & \textbf{46.67} \\

        Static-Sym ($w=5$)
        & 45.97
        & \textbf{0.87}
        & 45.96 \\
        \bottomrule
    \end{tabular}
\end{minipage}
\hfill
\begin{minipage}[t]{0.52\textwidth}
    \centering
    {\bfseries (b) Effect of thinking mode\par}
    \vspace{3pt}

    \setlength{\tabcolsep}{2.5pt}
    \renewcommand{\arraystretch}{1.10}
    \begin{tabular}{@{}llccc@{}}
        \toprule
        Model & Method
        & Acc $\uparrow$
        & Exceed $\downarrow$
        & Joint $\uparrow$ \\
        \midrule
        Instruct & GRPO
        & 62.99 & 0.70 & 62.76 \\

        & GDPO
        & 61.91 & 1.47 & 61.42 \\

        & DARA-Asym
        & \underline{62.17}
        & \underline{0.25}
        & \underline{62.11} \\

        & DARA-Sym
        & 60.48 & 0.01 & 60.48 \\
        \midrule
        Thinking & GRPO
        & 64.82 & 1.86 & 64.21 \\

        & GDPO
        & 63.68 & 1.44 & 63.41 \\

        & DARA-Asym
        & \underline{64.23}
        & \underline{0.94}
        & \underline{64.08} \\

        & DARA-Sym
        & 62.45 & 0.16 & 62.44 \\
        \bottomrule
    \end{tabular}
\end{minipage}

\vspace{-10pt}
\end{table*}

\FloatBarrier
\section{Conclusion}

We studied uneven learning in multi-reward RL through the relationship between reward activity and optimization signal. Under idealized GDPO normalization, we showed that a reward's batch-level advantage energy is proportional to its active-group density. Based on this analysis, we introduced DARA, which applies a density correction to strengthen signals from rewards that provide useful comparisons less frequently. Experiments on tool calling and mathematical reasoning show that DARA learns targeted behaviors faster and more stably, reaching high task performance in fewer training steps. These results highlight active-group density as a simple and measurable quantity for understanding and calibrating reward contributions in multi-reward RL.
\subsection*{AI use statement}

We used AI coding assistants to help implement training and evaluation
scripts, process experimental logs, and we used large
language models to polish the writing of this paper. The research questions,
theoretical analysis, and experimental design were developed by the authors.
All AI-assisted code was reviewed and tested by the authors, and all reported
results were checked against the raw experimental outputs. We take full
responsibility for the content of this paper.

\subsection*{Reproducibility statement}

The derivations of the density--energy relation and the density calibration
rule are given in Appendix~\ref{app:density_energy_proofs}. DARA changes only
the reward aggregation step of GDPO, as specified in
Equations~\ref{eq:density_weight}--\ref{eq:dara}. For tool calling,
Appendix~\ref{app:tool_calling_details} reports the training configuration,
reward definitions, number of training seeds, and BFCL-v4 evaluation protocol.
For mathematical reasoning, Appendix~\ref{app:math_details} reports the
training data, prompt format, reward definitions, advantage estimators,
evaluation benchmarks, decoding parameters, and metric definitions. We will
release our code and training configurations.

\bibliography{iclr2027_conference}
\bibliographystyle{iclr2027_conference}

\appendix
\raggedbottom
\section{Proofs for the Density--Energy Analysis}
\label{app:density_energy_proofs}

\subsection{Derivation of the Density--Energy Relation}
\label{app:proof_density_energy}

We provide the full derivation of Equation~\ref{eq:density_energy}.
For an active rollout group $i$, the sample variance of reward $k$ is

\begin{equation}
\sigma_{i,k}^{2}
=
\frac{1}{G-1}
\sum_{j=1}^{G}
\left(
r_k^{(i,j)}-\mu_{i,k}
\right)^2.
\label{eq:app_sample_variance}
\end{equation}

Using the reward-wise advantage $A_k^{(i,j)}$ defined in
Section~\ref{sec:method}, and omitting the numerical stabilizer
$\epsilon$ for the analytical derivation, we obtain

\begin{align}
\sum_{j=1}^{G}
\left(A_k^{(i,j)}\right)^2
&=
\frac{
\sum_{j=1}^{G}
\left(
r_k^{(i,j)}-\mu_{i,k}
\right)^2
}{
\sigma_{i,k}^{2}
}
\notag\\
&=
G-1.
\label{eq:app_group_energy}
\end{align}

Thus, every active group contributes exactly $G-1$ units of advantage
energy.

If reward $k$ is inactive in group $i$, all reward values within the group
are identical. Their centered rewards are therefore zero, and the
corresponding reward-wise advantages contribute zero energy.

From Equation~\ref{eq:group_activity},

\begin{equation}
\sum_{i=1}^{B} I_{i,k}
=
B\pi_k,
\label{eq:app_num_active_groups}
\end{equation}

so reward $k$ is active in exactly $B\pi_k$ groups.

Using the definition of advantage energy in
Equation~\ref{eq:advantage_energy}, we therefore have

\begin{align}
E_k
&=
\sum_{i=1}^{B}
\sum_{j=1}^{G}
\left(A_k^{(i,j)}\right)^2
\notag\\
&=
B\pi_k(G-1).
\label{eq:app_density_energy}
\end{align}

This recovers Equation~\ref{eq:density_energy}. Since reward-wise
normalization fixes the energy contributed by each active group to $G-1$,
the total advantage energy varies linearly with active-group density.

\subsection{Derivation of the Density Calibration}
\label{app:proof_density_calibration}

We next derive the uncapped density coefficient $\alpha_k^\star$ used in
Section~\ref{sec:dara}.

Suppose reward $k$ is multiplied by a scalar coefficient $\alpha_k$ before
aggregation. By the definition of advantage energy in
Equation~\ref{eq:advantage_energy}, its calibrated
advantage energy is

\begin{equation}
E_k(\alpha_k)
=
\alpha_k^2 E_k.
\label{eq:app_weighted_energy}
\end{equation}

Substituting the density--energy relation from
Equation~\ref{eq:density_energy} gives

\begin{equation}
E_k(\alpha_k)
=
\alpha_k^2
B\pi_k(G-1).
\label{eq:app_weighted_density_energy}
\end{equation}

Let

\begin{equation}
\pi_{\mathrm{ref}}
=
\max_{k:\pi_k>0}\pi_k
\label{eq:app_reference_density}
\end{equation}

denote the largest active-group density. Its corresponding advantage energy
is

\begin{equation}
E_{\mathrm{ref}}
=
B\pi_{\mathrm{ref}}(G-1).
\label{eq:app_reference_energy}
\end{equation}

To match the calibrated energy of reward $k$ to the reference energy, we
require

\begin{equation}
E_k(\alpha_k)
=
E_{\mathrm{ref}}.
\label{eq:app_energy_matching}
\end{equation}

Substituting Equations~\ref{eq:app_weighted_density_energy} and
\ref{eq:app_reference_energy} yields

\begin{equation}
\alpha_k^2
B\pi_k(G-1)
=
B\pi_{\mathrm{ref}}(G-1).
\end{equation}

Canceling the common factors gives

\begin{equation}
\alpha_k^2\pi_k
=
\pi_{\mathrm{ref}}.
\end{equation}

The positive solution is therefore

\begin{equation}
\alpha_k^\star
=
\sqrt{
\frac{\pi_{\mathrm{ref}}}{\pi_k}
}.
\label{eq:app_optimal_weight}
\end{equation}

This recovers the coefficient $\alpha_k^\star$ in Section~\ref{sec:dara}.

Finally, substituting $\alpha_k^\star$ into
Equation~\ref{eq:app_weighted_density_energy} gives

\begin{align}
E_k(\alpha_k^\star)
&=
\left(
\frac{\pi_{\mathrm{ref}}}{\pi_k}
\right)
B\pi_k(G-1)
\notag\\
&=
B\pi_{\mathrm{ref}}(G-1).
\label{eq:app_calibrated_energy}
\end{align}

This recovers the matched energy stated in Section~\ref{sec:dara}. Hence, before applying
the practical cap in Equation~\ref{eq:density_weight}, density calibration
removes the dependence of reward-wise advantage energy on the original
active-group density.

\begin{figure}[t]
    \centering
    \includegraphics[width=\textwidth]{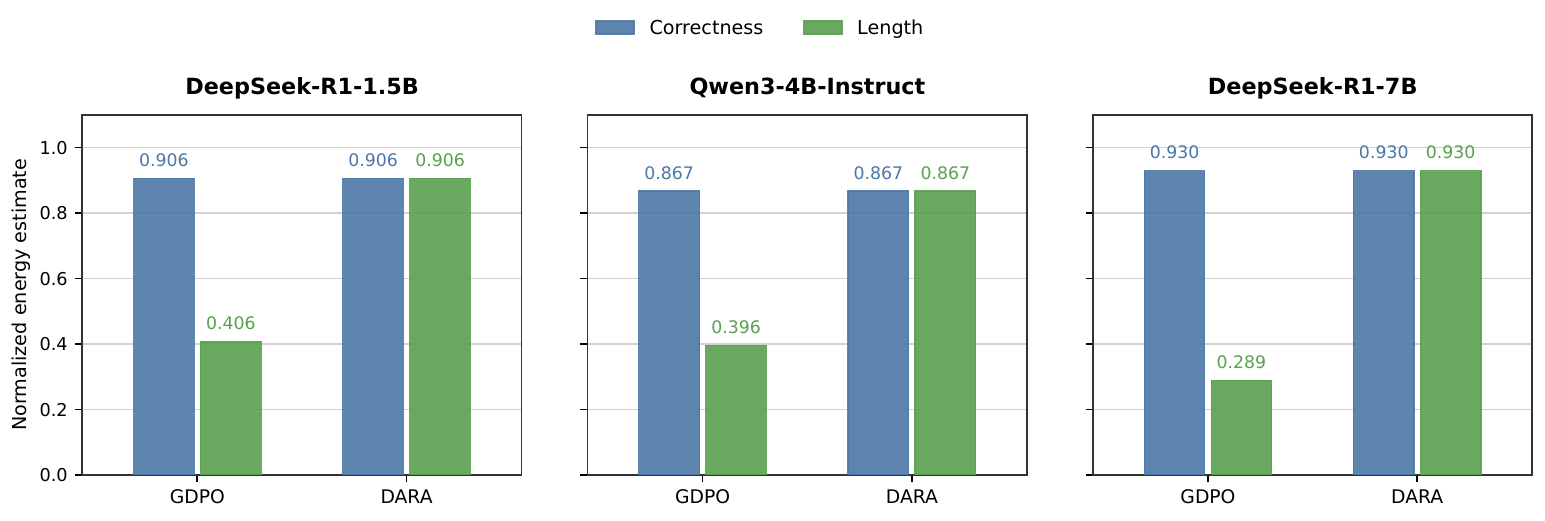}
    \caption{Reward-wise advantage energy on batches of mathematical
    reasoning training at step 40, before (GDPO) and after (DARA) density
    calibration. Values are from the recorded active-group densities
    and weights as $\pi_k$ and $w_k^2\pi_k$, which equal $E_k/[B(G-1)]$ by
    Equation~\ref{eq:density_energy}.}
    \label{fig:energy_logged}
\end{figure}

\begin{figure}[t]
    \centering
    \includegraphics[width=\textwidth]{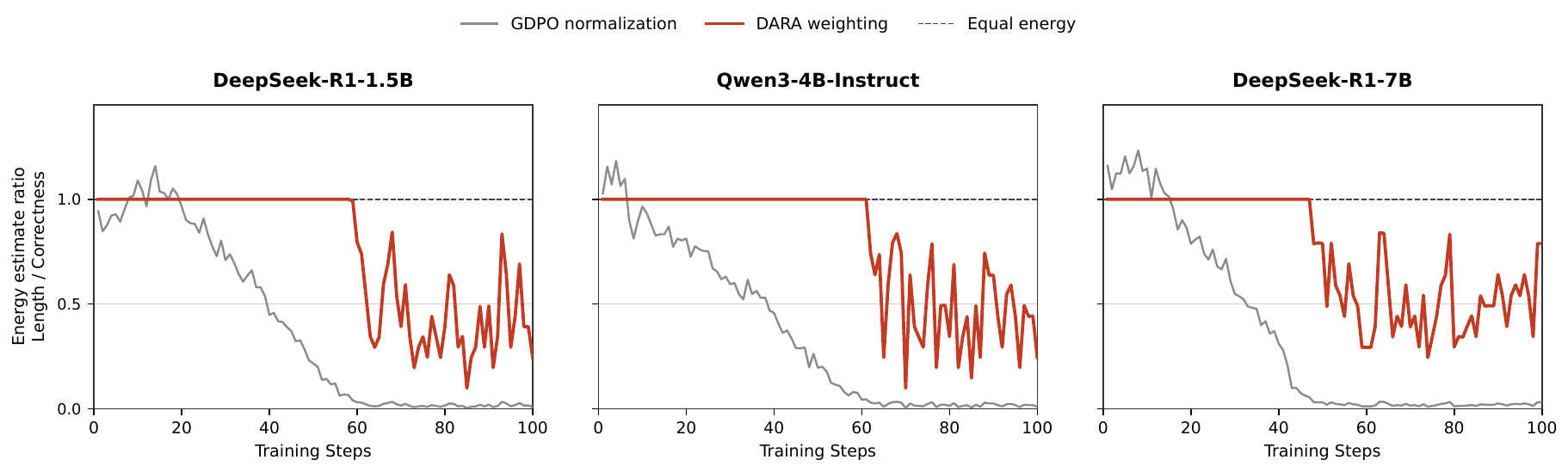}
    \caption{Ratio of the length to the correctness energy estimate on the same
    logged batches throughout training, without (GDPO normalization) and with
    (DARA weighting) density calibration. The dashed line marks equal energy.}
    \label{fig:energy_trajectory}
\end{figure}

\paragraph{Energy on training batches.}
Figure~\ref{fig:energy_logged} illustrates this calibration on logged batches
of DARA-Sym training in the mathematical reasoning experiments, where length is
the sparse reward. At step 40, the length reward is active in fewer groups than
the correctness reward, and under GDPO normalization its energy is 31--46\% of
that of correctness; DARA weighting matches the two.
Figure~\ref{fig:energy_trajectory} follows the length-to-correctness energy
ratio through training. Under GDPO normalization the ratio declines as length
compliance approaches saturation. DARA keeps the ratio at one until the
required length weight exceeds the cap $w_{\max}=5$, at steps 59, 62, and 48
for the three models. From these steps onward, the ratio stays below $0.04$
under GDPO normalization, whereas its median under DARA is 0.39--0.49. This
late phase corresponds to the strict compliance thresholds at which DARA and
GDPO separate most in Figure~\ref{fig:math_thresholds}.
\subsection{From Advantage Energy to Reward-Wise Policy Gradients}
\label{app:energy_gradient}

We now relate advantage energy to the policy-gradient contribution of each
reward. Where $E_k=B\pi_k(G-1)$ is used, we adopt the idealization of
Equation~\ref{eq:density_energy}; the other identities hold for the measured
$E_k$.

\paragraph{Setup.}
Fix the current policy and put $N=BG$. We analyze the reward-dependent
surrogate gradient at the rollout policy before clipping, with advantages and
loss weights held fixed during differentiation. Let $q_{ijt}$ be the
coefficient of token $t$ in the loss reduction, including its mask, and define
the effective response score and the gradient contributed by reward $k$ as
\begin{equation}
 z^{(i,j)}
 =N\sum_t q_{ijt}\nabla_\theta
       \log\pi_\theta\bigl(y_t^{(i,j)}\mid x_i,y_{<t}^{(i,j)}\bigr),
 \qquad
 g_k=\frac{1}{N}\sum_{i,j}A_k^{(i,j)}z^{(i,j)}.
 \label{eq:theory_effective_score}
\end{equation}
Sequence sums, per-response token means, and global token means correspond to
different choices of $q_{ijt}$.
Let $\mathcal F$ contain the prompts, all reward values, and the loss weights,
assume finite conditional second moments, and write
\begin{equation}
 m^{(i,j)}=\mathbb E[z^{(i,j)}\mid\mathcal F],\qquad
 \xi^{(i,j)}=z^{(i,j)}-m^{(i,j)},\qquad
 h_k=\frac{1}{N}\sum_{i,j}A_k^{(i,j)}m^{(i,j)}.
 \label{eq:theory_signal_noise}
\end{equation}
The conditional mean $h_k$ is the learning signal of reward $k$, and
$\xi^{(i,j)}$ is the fluctuation around it.
Flatten response indices to $a,b\in\{1,\ldots,N\}$, let $\mathbf A_k$ be the
vector of advantages of reward $k$, and define
\begin{equation}
 C_{ab}=\mathbb E[\langle\xi_a,\xi_b\rangle\mid\mathcal F],
 \qquad
 P=\operatorname{diag}(P_G,\ldots,P_G),\qquad
 P_G=I_G-\frac{\mathbf 1_G\mathbf 1_G^\top}{G}.
 \label{eq:theory_contrast_covariance}
\end{equation}
$C$ is positive semidefinite.
Since reward-wise normalization gives $\sum_jA_k^{(i,j)}=0$ in every group,
$P\mathbf A_k=\mathbf A_k$: any score component shared by all responses to a
prompt cancels exactly, and the remaining within-group dependence is kept in
$PCP$.

\begin{proposition}[Advantage energy and gradient geometry]
\label{prop:theory_gradient_geometry}
For any rewards $k$ and $l$,
\begin{equation}
 \mathbb E[g_k\mid\mathcal F]=h_k,\qquad
 \mathbb E[\langle g_k,g_l\rangle\mid\mathcal F]
 =\langle h_k,h_l\rangle
   +\frac{\mathbf A_k^\top PCP\mathbf A_l}{N^2}.
 \label{eq:theory_exact_geometry}
\end{equation}
If, for some $c>0$ and $0\leq\delta<1$, the covariance on group contrasts
satisfies
\begin{equation}
 \|PCP-cP\|_{\mathrm{op}}\leq\delta c,
 \label{eq:theory_covariance_condition}
\end{equation}
then
\begin{equation}
 \mathbb E[\langle g_k,g_l\rangle\mid\mathcal F]
 =\langle h_k,h_l\rangle
   +\frac{c}{N^2}\langle\mathbf A_k,\mathbf A_l\rangle+R_{kl},
 \qquad
 |R_{kl}|\leq\frac{\delta c}{N^2}\sqrt{E_kE_l}.
 \label{eq:theory_geometry_bound}
\end{equation}
In particular, the gradient fluctuation
$V_k=\mathbb E[\|g_k-h_k\|^2\mid\mathcal F]$ satisfies
\begin{equation}
 (1-\delta)\frac{cE_k}{N^2}\leq V_k
 \leq(1+\delta)\frac{cE_k}{N^2},\qquad
 \mathbb E[\|g_k\|^2\mid\mathcal F]=\|h_k\|^2+V_k.
 \label{eq:theory_norm_bounds}
\end{equation}
\end{proposition}

\begin{proof}
Substitute $z_a=m_a+\xi_a$ into Equation~\ref{eq:theory_effective_score}.
Conditional on $\mathcal F$, the advantages are fixed and every $\xi_a$ has
zero mean, so the mean--residual terms vanish and the residual--residual term
is $N^{-2}\mathbf A_k^\top C\mathbf A_l$. Inserting $P$ on both sides gives
Equation~\ref{eq:theory_exact_geometry}. With $\Delta=PCP-cP$, the remainder
is $R_{kl}=N^{-2}\mathbf A_k^\top\Delta\mathbf A_l$, and the operator-norm
bound together with $\|\mathbf A_k\|^2=E_k$ gives
Equation~\ref{eq:theory_geometry_bound}. Setting $l=k$ gives
Equation~\ref{eq:theory_norm_bounds}.
\end{proof}

Condition~\ref{eq:theory_covariance_condition} holds with $\delta=0$, for
example, when $\xi^{(i,j)}=b_i+e^{(i,j)}$, where the prompt-level terms $b_i$
are arbitrary and the response-level terms $e^{(i,j)}$ are conditionally
zero-mean, mutually uncorrelated, uncorrelated with the $b_i$, and have
common second moment $c$; the mean scores $m^{(i,j)}$ are unrestricted. If the
response-level second moments differ, the
residual term becomes the weighted energy
$N^{-2}\sum_{i,j}c_{ij}A_k^{(i,j)}A_l^{(i,j)}$. In the model with $\delta=0$,
the correlation between the gradient fluctuations of two rewards equals the
correlation of their advantage vectors,
\begin{equation}
 \frac{\mathbb E[\langle g_k-h_k,g_l-h_l\rangle\mid\mathcal F]}
 {\sqrt{V_kV_l}}
 =\frac{\langle\mathbf A_k,\mathbf A_l\rangle}{\sqrt{E_kE_l}},
 \label{eq:theory_residual_correlation}
\end{equation}
which can be computed directly from rollout rewards.

\begin{corollary}[Symmetric calibration and batch normalization]
\label{cor:theory_sym_normalization}
Under the density--energy identity, the capped DARA-Sym weights give
$w_k^2E_k=\min\{w_{\max}^2E_k,\,E_{\mathrm{ref}}\}$ with
$E_{\mathrm{ref}}=B\pi_{\mathrm{ref}}(G-1)$. Hence, under
Condition~\ref{eq:theory_covariance_condition}, every positive-density reward
whose cap does not bind satisfies
\begin{equation}
 (1-\delta)\frac{cE_{\mathrm{ref}}}{N^2}\leq w_k^2V_k
 \leq(1+\delta)\frac{cE_{\mathrm{ref}}}{N^2}.
 \label{eq:theory_calibrated_variance}
\end{equation}
Moreover, $\mathbf S_{\mathrm{sym}}=\sum_kw_k\mathbf A_k$ has
$\mu_{\mathcal B}(S_{\mathrm{sym}})=0$, and with
$Z_{\mathrm{sym}}=\sigma_{\mathcal B}(S_{\mathrm{sym}})+\epsilon$,
\begin{equation}
 \widehat g_{\mathrm{sym}}
 =\frac{1}{N}\sum_{i,j}\widehat A_{\mathrm{sym}}^{(i,j)}z^{(i,j)}
 =\frac{1}{Z_{\mathrm{sym}}}\sum_k w_kg_k .
 \label{eq:theory_sym_gradient}
\end{equation}
\end{corollary}

\begin{proof}
The first identity follows by multiplying the squared weight in
Equation~\ref{eq:density_weight} by $E_k=B\pi_k(G-1)$, and multiplying
Equation~\ref{eq:theory_norm_bounds} by $w_k^2$ gives
Equation~\ref{eq:theory_calibrated_variance}. Each $\mathbf A_k$ sums to zero
within every group, and so does $\mathbf S_{\mathrm{sym}}$; batch centering in
Equation~\ref{eq:dara} therefore subtracts zero, and linearity gives
Equation~\ref{eq:theory_sym_gradient}.
\end{proof}

\paragraph{The signal term for binary rewards.}
The signal $h_k$ in Proposition~\ref{prop:theory_gradient_geometry} is
unrestricted in general, but for a binary reward its average over the other
rewards has an explicit direction. Specialize to the sequence score
$z^{(i,j)}=\nabla_\theta\log\pi_\theta(y^{(i,j)}\mid x_i)$, write
$p_{i,k}=\Pr_\theta(r_k=1\mid x_i)$ and $f_{i,k}=G^{-1}\sum_jr_k^{(i,j)}$, and
let $\mathcal F_k\subseteq\mathcal F$ be generated by the prompts and the values
of reward $k$, so that $\mathbb E[g_k\mid\mathcal F_k]=\mathbb E[h_k\mid\mathcal F_k]$.

\begin{proposition}[Signal direction of a binary reward]
\label{prop:theory_binary_signal}
Given the prompts, suppose that all $BG$ responses are sampled independently
with $y^{(i,j)}\sim\pi_\theta(\cdot\mid x_i)$, that
$r_k^{(i,j)}=R_k(x_i,y^{(i,j)})\in\{0,1\}$ for a reward function $R_k$ that does
not depend on $\theta$, and that $0<p_{i,k}<1$. Then
\begin{equation}
 \mathbb E[g_k\mid\mathcal F_k]
 =\frac1N\sum_{i:\,0<f_{i,k}<1}
 \frac{G f_{i,k}(1-f_{i,k})}{(\sigma_{i,k}+\epsilon)\,p_{i,k}(1-p_{i,k})}\,
 \nabla_\theta p_{i,k},
 \label{eq:theory_binary_signal}
\end{equation}
and the positive part used by DARA-Asym satisfies
\begin{equation}
 \mathbb E\Bigl[\frac1N\sum_{i,j}\bigl[A_k^{(i,j)}\bigr]_+z^{(i,j)}
 \Bigm|\mathcal F_k\Bigr]
 =\frac1N\sum_{i:\,0<f_{i,k}<1}
 \frac{G f_{i,k}(1-f_{i,k})}{(\sigma_{i,k}+\epsilon)\,p_{i,k}}\,
 \nabla_\theta p_{i,k}.
 \label{eq:theory_binary_signal_asym}
\end{equation}
\end{proposition}

\begin{proof}
Write $\nu_{i,a}=\mathbb E[z\mid x_i,r_k=a]$ for $a\in\{0,1\}$. The score
identity $\mathbb E[z\mid x_i]=0$ gives $p_{i,k}\nu_{i,1}+(1-p_{i,k})\nu_{i,0}=0$,
and differentiating $p_{i,k}=\mathbb E[r_k\mid x_i]$ gives
$\nabla_\theta p_{i,k}=\mathbb E[r_kz\mid x_i]=p_{i,k}\nu_{i,1}$. Hence
$\nu_{i,1}=\nabla_\theta p_{i,k}/p_{i,k}$ and
$\nu_{i,1}-\nu_{i,0}=\nabla_\theta p_{i,k}/[p_{i,k}(1-p_{i,k})]$. Independence of
the responses gives $\mathbb E[z^{(i,j)}\mid\mathcal F_k]=\nu_{i,r_k^{(i,j)}}$.
In a mixed group, $A_k^{(i,j)}=(r_k^{(i,j)}-f_{i,k})/(\sigma_{i,k}+\epsilon)$,
so the advantage-weighted conditional means sum to
$Gf_{i,k}(1-f_{i,k})(\nu_{i,1}-\nu_{i,0})/(\sigma_{i,k}+\epsilon)$, and their
positive parts to $Gf_{i,k}(1-f_{i,k})\nu_{i,1}/(\sigma_{i,k}+\epsilon)$.
Constant groups have zero advantages. Substitution gives both identities.
\end{proof}

Averaged over the other rewards, every mixed group of a binary reward thus
contributes a positive multiple of the gradient that raises that reward's
success probability on its prompt, and the positive-only term of DARA-Asym
contributes a positive multiple of the same per-prompt gradient. Conditional
on $\mathcal F$, the DARA-Sym weight $w_k$ is fixed and scales $h_k$ and the
fluctuation $g_k-h_k$ equally, leaving $\|h_k\|^2/V_k$ unchanged. The identity
describes the direction of each group's contribution for fixed weights;
contributions from different prompts may still point in different directions.

Thus density calibration equalizes, up to the factors $1\pm\delta$, the gradient
fluctuation of every reward whose weight cap does not bind, and the final batch normalization of DARA-Sym rescales all
reward contributions by one common positive factor, preserving their relative
magnitudes. 

\subsection{Active-Group Density under Heterogeneous Prompts and Fixed Budgets}
\label{app:heterogeneous_density}

Fix the policy and a binary reward $k$, and write $p_{i,k}$ for its success
probability on prompt $x_i$. Let $q_G(p)=1-p^G-(1-p)^G$.

\begin{proposition}[Prompt-specific activity]
\label{prop:theory_heterogeneous_activity}
If the $G$ outcomes for prompt $i$ are conditionally independent with success
probability $p_{i,k}$, then
\begin{equation}
 \mathbb E[\pi_k\mid x_1,\ldots,x_B]
 =\frac1B\sum_i q_G(p_{i,k}),\qquad
 \mathbb E[E_k\mid x_1,\ldots,x_B]
 =(G-1)\sum_i q_G(p_{i,k}).
 \label{eq:theory_heterogeneous_density}
\end{equation}
Moreover, for $G\geq2$,
\begin{equation}
 q_{G+1}(p)-q_G(p)=p^G(1-p)+(1-p)^Gp\geq0,
 \qquad
 \frac1B\sum_i q_G(p_{i,k})\leq q_G\!\left(\frac1B\sum_i p_{i,k}\right),
 \label{eq:theory_monotonic_jensen}
\end{equation}
with strict monotonicity whenever $0<p<1$.
\end{proposition}

\begin{proof}
A group is inactive exactly when all outcomes succeed or all fail, with
probabilities $p_{i,k}^G$ and $(1-p_{i,k})^G$; linearity of expectation and
Equation~\ref{eq:density_energy} give
Equation~\ref{eq:theory_heterogeneous_density}. Subtraction gives the first
relation in Equation~\ref{eq:theory_monotonic_jensen}. Since
$q_G''(p)=-G(G-1)[p^{G-2}+(1-p)^{G-2}]\leq0$, $q_G$ is concave and Jensen's
inequality gives the second.
\end{proof}

Thus the group-size monotonicity holds prompt by prompt, without a common
success probability; substituting a pooled success rate overestimates density.
The monotonicity also holds without independence. For any joint distribution
of outcomes, the active probability is
$1-\Pr(r_1=\cdots=r_G=1)-\Pr(r_1=\cdots=r_G=0)$, and when groups of different
sizes are prefixes of one rollout process, a group that already contains both
outcomes remains active as responses are added, so the active probability is
nondecreasing in $G$.

\paragraph{Fixed response budget.}
The group-size study keeps the number of responses per rollout step fixed at
$N=BG$ (Table~\ref{tab:tool_calling_group_config}), so larger groups are
formed from fewer prompts. Let $\bar q_{k,G}=\mathbb E_X[q_G(p_k(X))]$ for
prompts drawn from a fixed distribution.

\begin{proposition}[Density and energy under a fixed budget]
\label{prop:theory_fixed_budget}
For group sizes $G\geq2$ dividing $N$,
\begin{equation}
 \mathbb E[\pi_k]=\bar q_{k,G},\qquad
 \frac{\mathbb E[E_k]}{N}=\left(1-\frac1G\right)\bar q_{k,G}.
 \label{eq:theory_fixed_budget}
\end{equation}
Both the expected active-group density and the expected advantage energy per
generated response are nondecreasing in $G$, and strictly increasing when
prompts with $0<p_k(X)<1$ have positive probability.
\end{proposition}

\begin{proof}
The expected number of active groups is $B\bar q_{k,G}$ and each active group
contributes energy $G-1$; substituting $B=N/G$ gives
Equation~\ref{eq:theory_fixed_budget}. By
Proposition~\ref{prop:theory_heterogeneous_activity}, $q_G(p)$ is
nondecreasing in $G$ for every $p$ and strictly increasing for $0<p<1$, and
averaging over prompts preserves both properties. Since $1-1/G$ is strictly
increasing, the same holds for the energy per response.
\end{proof}

Hence, at a fixed response budget, larger groups raise both the active-group
density and the advantage energy per generated response. Each active group of
size $G$ contributes energy $G-1$, so allocating the budget to larger groups
supplies more advantage energy per response even though fewer prompts are
sampled. This is the density prediction tested in
Figure~\ref{fig:group_size_format_convergence}; the number of prompts per step
for each $G$ is listed in Table~\ref{tab:tool_calling_group_config}.

\paragraph{Density along training.}
The density--energy identity holds for every realized batch. Along training, success rates change, and under
conditionally independent sampling $q_G$ is largest at $p=1/2$, so density
naturally peaks while a reward is being acquired. DARA instead changes
how density is mapped to advantage coefficients on a fixed rollout batch,
leaving $I_{i,k}$ and $\pi_k$ unchanged, and varying $G$ changes the activity
probability through the sampling allocation. Both therefore act on the
density--energy relation directly.

\subsection{Energy of Positive-Only Amplification}
\label{app:asym_geometry}

DARA-Asym is not covered by the exact energy matching of DARA-Sym, but its
energy admits simple bounds. Let
$E_k^+=\sum_{i,j}[A_k^{(i,j)}]_+^2$ and $Q_k^+=\sum_{i,j}[A_k^{(i,j)}]_+$.

\begin{proposition}[Energy of DARA-Asym]
\label{prop:theory_asym_energy}
For $\widetilde{\mathbf A}_{k,\mathrm{asym}}=\mathbf A_k+(w_k-1)[\mathbf A_k]_+$
with $w_k\geq1$,
\begin{equation}
 \|\widetilde{\mathbf A}_{k,\mathrm{asym}}\|^2=E_k+(w_k^2-1)E_k^+,
 \label{eq:theory_asym_energy}
\end{equation}
and both this energy and its value after batch centering,
$E_k+(w_k^2-1)E_k^+-(w_k-1)^2(Q_k^+)^2/N$, lie in $[E_k,\,w_k^2E_k]$.
\end{proposition}

\begin{proof}
Positive and negative entries have disjoint support, which gives
Equation~\ref{eq:theory_asym_energy}. The batch sum of the transformed
coefficients is $(w_k-1)Q_k^+$, so centering subtracts $(w_k-1)^2(Q_k^+)^2/N$.
The centered energy equals
$E_k+2(w_k-1)E_k^+ +(w_k-1)^2\bigl(E_k^+-(Q_k^+)^2/N\bigr)$, and the last
parenthesis is nonnegative by the Cauchy--Schwarz inequality, giving the lower
bound. Centering cannot increase the squared norm, and $E_k^+\leq E_k$ gives
the upper bound.
\end{proof}

Hence DARA-Asym raises the energy of a sparsely active reward by at most the
factor $w_k^2$ of DARA-Sym, and never below its GDPO energy, while
amplifying only the positive coefficients.

\section{Advantage Computation of Multi-Reward Methods}
\label{app:method_comparison}

\subsection{Comparison of Advantage Estimators}

Table~\ref{tab:method_comparison} summarizes how each method converts the
$K$ reward channels into the sequence-level advantage used in the clipped
policy objective. We use the notation of Section~\ref{sec:method}: for query
$i$ and response $j$, $A_k^{(i,j)}=(r_k^{(i,j)}-\mu_{i,k})/(\sigma_{i,k}+\epsilon)$
is the reward-wise group-normalized advantage, and
$\mathcal{N}_{\mathcal B}(S)=(S-\mu_{\mathcal B}(S))/(\sigma_{\mathcal B}(S)+\epsilon)$
denotes normalization over the rollout batch. The formulas for DVAO, SAW,
SA-MRPO, and GD$^2$PO-Hard follow the original papers
\citep{dvao,saw,samrpo,gd2po}; $w_k$ denotes prescribed objective weights where
a method uses them.

\begin{table}[htbp]
\caption{Advantage computation of multi-reward methods. The second column
gives the statistic that determines each reward's contribution, and the third
column gives the sequence-level advantage.}
\label{tab:method_comparison}
\vspace{4pt}
\centering
\small
\renewcommand{\arraystretch}{1.35}
\begin{tabularx}{\textwidth}{@{}l>{\raggedright\arraybackslash}p{0.27\textwidth}X@{}}
\toprule
Method & Reward weighting & Advantage \\
\midrule
GRPO
& Rewards summed before normalization
& $\displaystyle \frac{r_{\mathrm{sum}}^{(i,j)}-\mu_{i,\mathrm{sum}}}{\sigma_{i,\mathrm{sum}}+\epsilon},\quad r_{\mathrm{sum}}^{(i,j)}=\sum_k r_k^{(i,j)}$ \\
GDPO
& Equal weights on normalized rewards
& $\mathcal{N}_{\mathcal B}\bigl(\sum_k A_k^{(i,j)}\bigr)$ \\
DVAO
& Within-group reward standard deviation
& $\displaystyle \sum_k \tilde w_{i,k}A_k^{(i,j)},\quad \tilde w_{i,k}=\frac{w_k\sigma_{i,k}}{\sum_l w_l\sigma_{i,l}}$ \\
SAW
& Batch-level coefficient of variation $\mathrm{CV}_k$ of the offset reward $r_k-r_k^{\min}+\delta$
& $\displaystyle \mathcal{N}_{\mathcal B}\bigl(\sum_k \omega_k A_k^{(i,j)}\bigr),\quad \omega_k=K\,\frac{\mathrm{CV}_k}{\sum_l \mathrm{CV}_l}$ \\
SA-MRPO
& Batch-level saturation $s_k=\frac{\bar r_k-r_k^{\min}}{r_k^{\max}-r_k^{\min}}$
& $\mathcal{N}_{\mathcal B}\bigl(\sum_k \tilde w_k A_k^{(i,j)}\bigr),\quad \tilde w_k=w_k(1-s_k)^{\gamma}$ \\
GD$^2$PO-Hard
& Sign consistency of $\{A_k^{(i,j)}\}_k$
& $\mathcal{N}_{\mathcal B}\bigl(\sum_k A_k^{(i,j)}\bigr)$ over retained responses with
$\delta^{(i,j)}=\mathbf{1}[\text{no sign conflict among }A_k^{(i,j)}]$; the loss of query $i$ is scaled by $\hat\kappa_i=\frac{1}{G}\sum_j\delta^{(i,j)}$ \\
\midrule
DARA-Sym
& Active-group density $\pi_k$
& $\mathcal{N}_{\mathcal B}\bigl(\sum_k w_k A_k^{(i,j)}\bigr),\quad w_k=\min\bigl(w_{\max},\sqrt{\pi_{\mathrm{ref}}/\pi_k}\bigr)$ \\
DARA-Asym
& Active-group density $\pi_k$
& $\mathcal{N}_{\mathcal B}\bigl(\sum_k A_k^{(i,j)}+(w_k-1)[A_k^{(i,j)}]_+\bigr)$ \\
\bottomrule
\end{tabularx}
\end{table}

DVAO, SAW, and SA-MRPO adapt reward weights to the dispersion or attainment
level of each reward, and GD$^2$PO-Hard removes responses whose reward-wise
advantages conflict. DARA instead calibrates each reward by how often it
provides a nonzero comparison within the batch, which directly determines its
advantage energy under reward-wise normalization
(Equation~\ref{eq:density_energy}).

\subsection{Pseudocode}

Algorithm~\ref{alg:dara} gives the advantage computation of DARA for one rollout
batch with $K$ rewards. The resulting advantages replace the GDPO advantages;
the policy objective is unchanged.

\begin{algorithm}[H]
\caption{DARA advantage computation for one rollout batch}
\label{alg:dara}
\begin{algorithmic}[1]
\Require rewards $r_k^{(i,j)}$ for $i=1,\ldots,B$, $j=1,\ldots,G$, $k=1,\ldots,K$; cap $w_{\max}$; threshold $\tau$; variant $v\in\{\mathrm{sym},\mathrm{asym}\}$
\For{$k=1,\ldots,K$}
  \For{$i=1,\ldots,B$}
    \State $\mu_{i,k}\gets\frac{1}{G}\sum_{j}r_k^{(i,j)}$, \ $\sigma_{i,k}\gets\mathrm{std}_j\bigl(r_k^{(i,j)}\bigr)$
    \State $A_k^{(i,j)}\gets\bigl(r_k^{(i,j)}-\mu_{i,k}\bigr)/(\sigma_{i,k}+\epsilon)$ for all $j$
    \State $I_{i,k}\gets\mathbf{1}\bigl[\sum_{j}|A_k^{(i,j)}|>\tau\bigr]$
  \EndFor
  \State $\pi_k\gets\frac{1}{B}\sum_{i}I_{i,k}$ \Comment{active-group density}
\EndFor
\State $\pi_{\mathrm{ref}}\gets\max_k\pi_k$
\For{$k=1,\ldots,K$}
  \State $w_k\gets\min\bigl(w_{\max},\sqrt{\pi_{\mathrm{ref}}/\pi_k}\bigr)$ if $\pi_k>0$, otherwise $w_k\gets1$
\EndFor
\For{each response $(i,j)$}
  \If{$v=\mathrm{sym}$}
    \State $S^{(i,j)}\gets\sum_k w_kA_k^{(i,j)}$
  \Else
    \State $S^{(i,j)}\gets\sum_k\bigl(A_k^{(i,j)}+(w_k-1)\max(A_k^{(i,j)},0)\bigr)$
  \EndIf
\EndFor
\State $\widehat A^{(i,j)}\gets\bigl(S^{(i,j)}-\mu_{\mathcal B}(S)\bigr)/\bigl(\sigma_{\mathcal B}(S)+\epsilon\bigr)$ for all $(i,j)$
\State \Return $\widehat A^{(i,j)}$, assigned to every generated token of response $(i,j)$
\end{algorithmic}
\end{algorithm}

\section{Tool-Calling Experimental Details}
\label{app:tool_calling_details}

\subsection{Implementation and Optimization Details}

We implement all methods in the same verl-based training stack derived from
the released GDPO codebase. The model, rollout engine, optimizer settings, and
training budget are matched across methods. Table~\ref{tab:tool_calling_training_config}
lists the complete configuration used for the main tool-calling experiments.

\begin{table}[H]
    \caption{Training configuration for the tool-calling experiments.}
    \label{tab:tool_calling_training_config}
    \vspace{4pt}
    \centering
    \small
    \setlength{\tabcolsep}{5.5pt}
    \renewcommand{\arraystretch}{1.08}
    \begin{tabular}{@{}ll@{}}
        \toprule
        Setting & Value \\
        \midrule
        Backbones
        & Qwen2.5-1.5B-Instruct / Qwen2.5-3B-Instruct \\
        Training / rollout
        & FSDP / vLLM, tensor parallelism 1, bfloat16 \\
        Training / validation / checkpoint
        & 100 steps / every 10 steps / step 100 \\
        Prompts / responses per prompt / total
        & 512 / 4 / 2,048 \\
        PPO minibatch / microbatch / epochs
        & 512 responses / 256 responses / 1 \\
        Learning rate / gradient clip / PPO clip
        & $10^{-6}$ / 1.0 / 0.2 \\
        Entropy coefficient
        & 0.001 \\
        Maximum prompt / response length
        & 2,048 / 1,024 tokens \\
        Rollout sampling
        & temperature 1.0, top-$p$ 1.0, top-$k$ $-1$ \\
        Dynamic token budget / rollout memory
        & 16,384 tokens per GPU / 0.6 \\
        Gradient checkpointing / padding removal
        & True / True \\
        Training seeds
        & 5 (main); 3 (rollout-group and three-reward analyses) \\
        \bottomrule
    \end{tabular}
\end{table}

The actor parameter, gradient, and optimizer states remain on device, while
the reference parameters are offloaded. The 3B runs additionally enable
expandable CUDA allocator segments. Each method is validated with the same
checkpoints and decoding configuration; only its reward aggregation rule is
changed.

\subsection{BFCL-v4 Evaluation Protocol}

We evaluate each final checkpoint once on every BFCL-v4 case through the same
ToolRL prompting and response-parsing interface. The official BFCL checkers
score the generated function names and arguments for Live and Non-Live tasks
and the resulting tool-execution states and dialogue outcomes for Multi-Turn
tasks. The three subset scores are averaged equally to obtain the reported
Average. We compute all metrics independently for each training run and report
the mean across training seeds; the base model is evaluated once.

\begin{table}[H]
    \caption{Inference and evaluation configuration for BFCL-v4.}
    \label{tab:tool_calling_evaluation_config}
    \vspace{4pt}
    \centering
    \small
    \setlength{\tabcolsep}{5.5pt}
    \renewcommand{\arraystretch}{1.08}
    \begin{tabular}{@{}ll@{}}
        \toprule
        Setting & Value \\
        \midrule
        Evaluator
        & BFCL-v4, gorilla commit \texttt{6ea57973} \\
        Subsets
        & Live, Non-Live, and Multi-Turn \\
        Trajectories
        & One per evaluation case \\
        Temperature / top-$p$
        & 0.6 / 0.95 \\
        Maximum new tokens / context length
        & 8,192 / 32,768 \\
        Inference seed / top-$k$
        & 0 / $-1$ \\
        Inference engine
        & vLLM 0.11.0, tensor parallelism 1, bfloat16 \\
        Reported aggregation
        & Mean of Live, Non-Live, and Multi-Turn \\
        \bottomrule
    \end{tabular}
\end{table}

Format is evaluated on the complete assistant output. A response receives a
score of one when it follows an accepted sequence of \texttt{<think>},
\texttt{<tool\_call>}, and \texttt{<response>} blocks with the required order,
closing tags, and line breaks, and zero otherwise. When a case contains
multiple assistant outputs, their format scores are averaged within the case
before applying the same category and subset aggregation used for accuracy.
Average Format is the arithmetic mean of the Live, Non-Live, and Multi-Turn
format scores.

\subsection{Additional Analysis Configurations}

For the rollout-group analysis, we vary $G$ while keeping 2,048 responses and
four optimizer updates per training step. The prompt batch and PPO batch sizes
are adjusted together as shown in Table~\ref{tab:tool_calling_group_config}; all
other settings are identical to the main experiments.

\begin{table}[H]
    \caption{Batch configuration for the rollout-group analysis.}
    \label{tab:tool_calling_group_config}
    \vspace{4pt}
    \centering
    \small
    \setlength{\tabcolsep}{5.5pt}
    \renewcommand{\arraystretch}{1.08}
    \begin{tabular}{@{}rrrrr@{}}
        \toprule
        $G$ & Prompts & PPO mini & PPO micro & Responses \\
        \midrule
        4  & 512 & 128 & 64 & 2,048 \\
        8  & 256 & 64  & 32 & 2,048 \\
        16 & 128 & 32  & 16 & 2,048 \\
        32 & 64  & 16  & 8  & 2,048 \\
        \bottomrule
    \end{tabular}
\end{table}

The three-reward experiments add a length reward to correctness and format.
Let $n$ be the number of whitespace-delimited words inside the extracted
\texttt{<think>} segment. The reward is
$r_{\mathrm{length}}=\mathbb{I}[n\leq 16]$, with a score of zero when
the required thinking tags are absent. These runs use
$G=4$ and otherwise retain the main configuration. On BFCL-v4, Len.\ is the
length reward averaged with the same output-to-case and subset aggregation as
Format.

\section{Mathematical Reasoning Experimental Details}

\label{app:math_details}

\subsection{Implementation and Optimization Details}

We implement all mathematical-reasoning experiments in the same verl-based
training stack. The actor is optimized with FSDP, while responses are generated
with vLLM. Our training set is the full DeepScaleR-Preview collection. Each example is
represented as a single user turn containing the original problem followed by
``Please reason step by step, and put your final answer within
\texttt{\textbackslash boxed\{\}}.'' Prompts exceeding 1,024
tokens after applying the backbone-specific chat template are removed. A
seeded 200-example monitoring set is sampled from the training collection
without removing those examples from training; it is used only for online
validation and is not used as an independent test benchmark. Table~\ref{tab:math_training_config} lists the shared training configuration.
The 1.5B and 7B experiments use DeepSeek-R1-1.5B and
DeepSeek-R1-7B, the 4B experiments use Qwen3-4B-Instruct,
and the thinking-mode comparison uses Qwen3-4B-Thinking. Unless explicitly stated otherwise, all
method comparisons at a given scale use the same backbone, data, random seed,
training budget, and throughput configuration. All reported results use
checkpoints from the first 100 updates.

\begin{table}[htbp]

\caption{Training configuration for the mathematical-reasoning experiments.}

\label{tab:math_training_config}

\vspace{4pt}

\centering

\small

\setlength{\tabcolsep}{5.5pt}

\renewcommand{\arraystretch}{1.08}

\begin{tabular}{@{}ll@{}}

    \toprule

    Setting & Value \\

    \midrule

    Backbones

    & DeepSeek-R1-Distill-Qwen-1.5B / 7B \\

    & Qwen3-4B-Instruct-2507 / Qwen3-4B-Thinking-2507 \\

    Training data

    & DeepScaleR-Preview, full training collection \\

    Training / rollout

    & FSDP / vLLM, rollout tensor parallelism 1 \\

    Prompt / response length

    & 1,024 / 8,000 tokens \\

    Accepted prompt groups / candidate groups

    & 512 / 768 per generation batch \\

    Responses per prompt, $G$

    & 16 \\

    PPO minibatch / epochs

    & 64 / 1 \\

    Learning rate / gradient clip

    & $10^{-6}$ / 1.0 \\

    PPO clipping, lower / upper

    & 0.20 / 0.28 \\

    KL loss coefficient / type

    & 0.0005 / MSE \\

    Entropy coefficient

    & 0 \\

    Loss aggregation

    & Token mean \\

    Rollout sampling

    & temperature 1.0, top-$p$ 1.0 \\

    Dynamic actor token budget

    & 16,384 tokens per GPU \\

    vLLM batched-token limit

    & 32,768 \\

    Gradient checkpointing / padding removal

    & True / True \\

    Actor precision

    & FP32 master parameters; BF16 forward/backward; FP32 reduction \\

    Reference precision / parameter offload

    & BF16 / True \\

    \bottomrule

\end{tabular}

\end{table}

The actor retains FP32 master parameters while using BF16 parameters for
forward and backward computation; gradient reductions are performed in FP32.
The reference model uses BF16 parameters and parameter offloading. Actor
checkpoints contain model, optimizer, auxiliary trainer state, and a Hugging
Face model export, permitting exact continuation of interrupted runs. The
rollout memory fraction is 0.9 for the 1.5B configuration and 0.8 for the 4B
configuration; the corresponding maximum numbers of concurrent sequences are
256 and 160. We use dynamic group filtering to avoid optimizer updates from prompt groups
that carry no within-group learning signal. Candidate prompt groups are drawn
in batches of 768, each with 16 responses. A group is retained when its summed
sequence reward is not constant across all responses. Sampling continues until
512 informative groups are collected or the generation-batch cap is reached.
The cap is two batches for the original 1.5B configuration and four batches for
the 4B configuration and resumed runs. All compared methods use the same
retained groups within a run configuration.

\subsection{Reward Construction and Advantage Estimators}

Training uses two binary reward channels. Let $y$ be the reference answer,
$\hat{y}$ the generated response, and $L$ the number of valid response tokens.
The correctness reward is

\begin{equation}
    r_{\mathrm{correct}} =
    \mathbb{I}\left[\operatorname{MathGrader}(\hat{y},y)=1\right],
\end{equation}

where \textsc{MathGrader} is the same boxed-answer equivalence grader used by
the verl Hendrycks-MATH reward implementation. The length reward is

\begin{equation}
    r_{\mathrm{length}} = \mathbb{I}[L\leq 4000].
\end{equation}

The scalar reward exposed to estimators that do not consume separate channels
is $r_{\mathrm{score}}=r_{\mathrm{correct}}+r_{\mathrm{length}}$. Token lengths
are taken directly from the valid response mask rather than estimated from the
decoded string, so the 4,000-token boundary is identical in reward computation
and evaluation.

For completeness, we describe the advantage estimators used in the main
comparisons. For prompt group $i$, response $g$, and reward channel $k$, let
$r_{ig}^{(k)}$ denote the channel reward and define its group-normalized
advantage as

\begin{equation}
    a_{ig}^{(k)} =
    \frac{r_{ig}^{(k)}-\mu_i^{(k)}}{\sigma_i^{(k)}+\epsilon},
\end{equation}

where $\mu_i^{(k)}$ and $\sigma_i^{(k)}$ are computed over the $G=16$
responses to prompt $i$. GRPO applies this normalization once to the fused
scalar reward. GDPO instead normalizes correctness and length separately,
sums the two channel advantages with equal coefficients, and applies a final
masked whitening operation over the training batch. DARA estimates how frequently each reward channel is informative. Let
$\pi_k$ be the fraction of prompt groups for which channel $k$ has nonzero
advantage energy, using a numerical activity threshold $\tau=10^{-8}$. With
$\pi_{\mathrm{ref}}=\max_k\pi_k$, the channel weight is

\begin{equation}
    w_k =
    \begin{cases}
        \min\left(w_{\max}, \sqrt{\pi_{\mathrm{ref}}/\pi_k}\right), & \pi_k>0,\\
        1, & \pi_k=0,
    \end{cases}
    \qquad w_{\max}=5.
\end{equation}

DARA-Asymmetric applies this weight only to positive advantages,

\begin{equation}
    \tilde{a}_{ig} = \sum_k
    \left(w_k[a_{ig}^{(k)}]_+ + [a_{ig}^{(k)}]_-\right),
\end{equation}
where $[a]_+=\max(a,0)$ and $[a]_-=\min(a,0)$,

whereas DARA-Symmetric applies the same weight to both signs,

\begin{equation}
    \tilde{a}_{ig}^{\mathrm{sym}} = \sum_k w_k a_{ig}^{(k)}.
\end{equation}

In both cases, the combined response-level values are normalized across the
batch using the sample standard deviation and are then broadcast to all valid
tokens of the corresponding response. Thus, the asymmetric and symmetric variants
differ only in whether density correction scales negative advantages. The
Static-Asym and Static-Sym controls retain the corresponding DARA pipeline but
fix the correctness and length weights to 1 and 5, respectively.

\subsection{Mathematical-Reasoning Evaluation Protocol}

We evaluate checkpoints on five public mathematical-reasoning benchmarks:
MATH-500, AIME 2024, AMC 2022/2023, Minerva Math, and the English competition
subset of OlympiadBench. Table~\ref{tab:math_evaluation_config} gives the exact
data sources and inference configuration. MATH-500 is used as the MATH
component. The complete suite contains 1,559 questions and therefore 24,944
generations per checkpoint and evaluation seed. The evaluation prompt exactly follows the training prompt: the original problem
and boxed-answer instruction are placed in one user turn, no system message is
added, and the checkpoint tokenizer supplies the generation prompt through its
native chat template. Ground-truth rendering is normalized without modifying
the grader: integral floating-point AMC answers are converted to integer
strings, the first answer is selected when OlympiadBench supplies a list, and
outer LaTeX math delimiters are removed from OlympiadBench answers.

\begin{table}[H]

\caption{Inference and evaluation configuration for mathematical reasoning.}

\label{tab:math_evaluation_config}

\vspace{4pt}

\centering

\small

\setlength{\tabcolsep}{5.5pt}

\renewcommand{\arraystretch}{1.08}

\begin{tabular}{@{}ll@{}}

    \toprule

    Setting & Value \\

    \midrule

    MATH

    & MATH-500, test split, 500 questions \\

    AIME

    & AIME 2024, 30 questions \\

    AMC

    & AMC 2022/2023, 83 questions \\

    Minerva

    & Minerva Math, test split, 272 questions \\

    OlympiadBench

    & \texttt{OE\_TO\_maths\_en\_COMP}, 674 questions \\

    Samples per question

    & 16 \\

    Temperature / top-$p$

    & 0.6 / 0.95 \\

    Maximum new tokens / context length

    & 32,768 / 34,816 \\

    Inference engine

    & vLLM, tensor parallelism 2, bfloat16 \\

    vLLM memory fraction

    & 0.85 \\

    Evaluated checkpoints

    & Steps $10,20,\ldots,100$ \\

    Generations per checkpoint and seed

    & 24,944 \\

    Exceed threshold

    & More than 4,000 returned tokens \\

    Bootstrap confidence intervals

    & 2,000 question-level resamples \\

    Reported benchmark aggregation

    & Unweighted mean over the five benchmarks \\

    \bottomrule

\end{tabular}

\end{table}

For the checkpoint-wise early-training analysis, we evaluate steps
$10,20,\ldots,100$.

\subsection{Metrics and Aggregation}

For a benchmark with $Q$ questions and $G=16$ sampled responses per question,
strict accuracy is

\begin{equation}
    \operatorname{Acc}_{\mathrm{strict}} =
    \frac{1}{Q}\sum_{q=1}^{Q}\frac{1}{G}
    \sum_{g=1}^{G}\mathbb{I}[\operatorname{MathGrader}(\hat{y}_{qg},y_q)=1].
\end{equation}

This is the equal-question-weighted mean sample accuracy, not the probability
that at least one of 16 responses is correct. Our primary accuracy
uses a conservative numerical-equivalence fallback:

\begin{equation}
    \operatorname{correct}_{qg} =
    \operatorname{strict}_{qg}\ \lor
    \operatorname{NumEq}(\hat{y}_{qg},y_q;\operatorname{rtol}=0).
\end{equation}

The fallback is invoked only when the strict grader returns zero and therefore
can only rescue, never reverse, a strict match. It handles exact numerical
rendering differences but does not grant equivalence to general symbolic,
interval, radical, or multi-answer expressions. Strict and fallback-adjusted
accuracies are both retained in the raw summaries.

Exceed measures the fraction of all generated responses whose vLLM-returned
token count is strictly greater than 4,000:

\begin{equation}
    \operatorname{Exceed} = \frac{1}{QG}
    \sum_{q=1}^{Q}\sum_{g=1}^{G}\mathbb{I}[L_{qg}>4000].
\end{equation}

We additionally report Joint Success,

\begin{equation}
    \operatorname{Joint} = \frac{1}{QG}
    \sum_{q=1}^{Q}\sum_{g=1}^{G}
    \mathbb{I}[\operatorname{correct}_{qg}=1\ \land\ L_{qg}\leq4000],
\end{equation}

which requires both correctness and satisfaction of the training-time length
criterion. Responses truncated at the 32,768-token generation limit remain in
the denominator, receive no special exclusion, and are scored from their
returned text. For each benchmark and evaluation seed, 95\% confidence
intervals are obtained from 2,000 nonparametric resamples of questions. The
reported Macro Accuracy, Macro Exceed, and Macro Joint Success are unweighted
arithmetic means.

\end{document}